\documentclass[twoside,11pt]{article}

\usepackage{blindtext}

\usepackage[preprint]{jmlr2e}
\usepackage{amsmath,amsfonts,amssymb}
\usepackage{bbm}
\usepackage{enumitem}
\usepackage{floatrow} 
\usepackage{subcaption}

\newcommand{\Rlogo}{\protect\includegraphics[height=1.8ex,keepaspectratio]{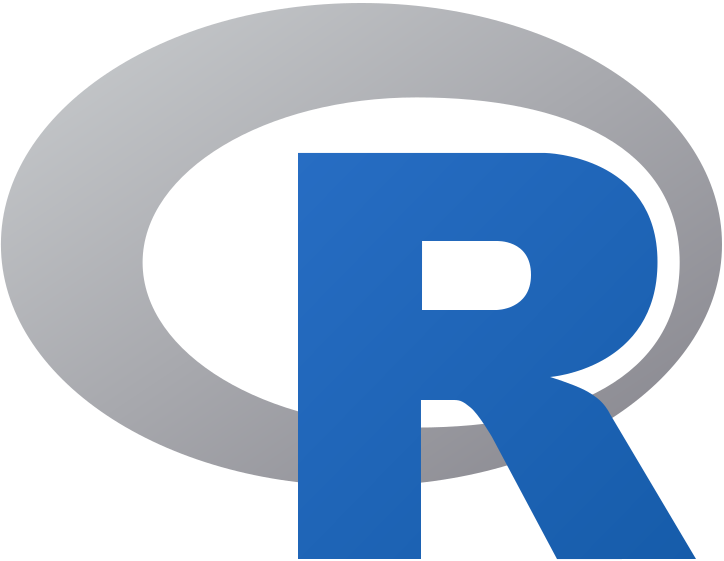}}
\newcommand{\cmmnt}[1]{}
\newcommand{\indep}{\perp \!\!\! \perp}

\usepackage{lastpage}
\jmlrheading{}{}{}{}{}{}{Christos Revelas, Otilia Boldea and Bas J.M. Werker}

\ShortHeadings{When does Subagging Work?}{Revelas, Boldea and Werker}
\firstpageno{1}

\begin{document}

\title{When does Subagging Work?}

\author{\name Christos Revelas \email c.revelas@tilburguniversity.edu\\
	\name Otilia Boldea \email o.boldea@tilburguniversity.edu\\
	\name Bas J.M. Werker \email b.j.m.werker@tilburguniversity.edu\\
       \addr Tilburg University\\
       Department of Econometrics and Operations Research\\
       Warandelaan 2, 5037 AB Tilburg, Netherlands}

\editor{My editor}

\maketitle

\begin{abstract}% 
We study the effectiveness of subagging, or subsample aggregating, on regression trees, a popular non-parametric method in machine learning. First, we give sufficient conditions for pointwise consistency of trees.
We formalize that (i) the bias depends on the diameter of cells, hence trees with few splits tend to be biased, and (ii) the variance depends on the number of observations in cells, hence trees with many splits tend to have large variance. While these statements for bias and variance are known to hold globally in the covariate space, we show that, under some constraints, they are also true locally.
Second, we compare the performance of subagging to that of trees across different numbers of splits.
We find that (1) for any given number of splits, subagging improves upon a single tree, and (2) this improvement is larger for many splits than it is for few splits. However, (3) a single tree grown at optimal size can outperform subagging if the size of its individual trees is not optimally chosen. This last result goes against common practice of growing large randomized trees to eliminate bias and then averaging to reduce variance. 
\end{abstract}


\section{Introduction}

Decision trees are a popular method for non-parametric regression estimation in statistics, machine learning, economics and data science practice. Given a dataset, a decision tree partitions the covariate space and estimates the regression function locally, i.e., in each cell of the partition, by a simple parametric form, typically a mean. CART\footnote{ Short for classification and regression trees, introduced by \citealp{breiman1984cart}. In the present paper, we refer to CART for \textit{regression} trees constructed based on Breiman et al.'s methodology.} is a prominent tree construction methodology. Bagging\footnote{ Bagging, introduced for CART by \citealp{breiman1996bagging}, is a blend word for \textbf{b}ootstrap \textbf{agg}regat\textbf{ing}.} is a randomization technique that consists of averaging multiple trees grown on bootstrap samples of the dataset. Early experiments suggested that bagging can improve the accuracy of learners, and in particular trees. On the one hand, this led to the development of randomization variants and extensions, such as random forests\footnote{ Random forests, introduced by \citealp{breiman2001random}, are averages of trees grown on bootstrap samples with an additional randomization of the covariate selection at each split.}. On the other hand, it led to many studies trying to \textit{clarify whether and why randomization methods ``work"}. 

\citealp{breiman1996bagging} argued that bagging mimics having several datasets available for estimation and heuristically showed that bagging can improve upon unstable\footnote{ Defined in \citealp{breiman1996bagging} as learners for which small changes in the data can lead to large changes in predictions: in other words, estimators with high variance.} learners, such as trees, through \textit{aggregation}. These arguments hold on average over covariates. In his simulations, trees are fully grown and then pruned\footnote{ Pruning, introduced in \citealp{breiman1984cart}, is a procedure consisting of merging one-by-one the terminal nodes of a fully grown tree and choosing as final tree the one with smallest out-of-sample mean-squared error.}, both for the single tree and the bagged estimator. Improvement is measured as a reduction in average, over covariates, out-of-sample mean-squared error. \citealp{buhlmann2002analyzing} were the first to establish rigorous results for subagging in the context of trees. They argued that instability\footnote{ Defined in \citealp{buhlmann2002analyzing} as predictors that have non-zero asymptotic variance. Stability is defined locally in the covariates, contrary to \citealp{breiman1996bagging} where, implicitly, stability refers to global, i.e., on average over covariates, low variance. } comes from hard decisions, such as indicators, and that randomization helps in reducing the variance through \textit{smoothing} of hard decisions. They prove that subagging reduces mean-squared error locally around the split point in the case of stumps, i.e., trees with a single split. In their simulations, trees are grown large and without pruning, and improvement is also measured as a reduction in average mean-squared error. What these two studies have in common, is that they compare trees with (su)bagging trees of similar size\footnote{ The \textit{size} of a tree is defined as the number of terminal nodes, also known as leaves, or cells. The size of a tree is therefore equal to one plus the number of splits. Tree size is not to be confused with tree \textit{depth}, defined as the maximum distance between the root and a terminal node. A tree of depth $\delta$ can have up to $2^{\delta}$ cells, and two trees of same depth can have different sizes.}. Breiman compares pruned trees with bagging of pruned trees. Bühlman and Yu at first compare stumps with subagged stumps, then large trees with subagging of large trees. What neither of these two studies does, is to compare trees with (su)bagging \textit{across} sizes. For example, \textit{how does (su)bagging small trees compare to a single large tree, and vice versa?} Moreover, both studies place tree instability at the centre of (su)bagging's performance, but, \textit{are trees always unstable? If not, how does (su)bagging perform for stable trees?}

There has been a tendency to grow large trees, i.e., with many splits, when considering ensemble methods such as (su)bagging and random forests\footnote{ See e.g. \citealp{hastie2009elements} and \citealp{james2013introduction}.}$^,$\footnote{See also the \textbf{randomForest} package %by \citealp{randomForestR}
in \Rlogo, version 4.7-1.1: the default minimum node size is five observations, and by default there is no bound on the number of nodes allowed.}, the idea being that by growing large trees, bias is eliminated, and by averaging, the variance is reduced. This idea goes back to \citealp{breiman2001random}, who suggested that in a forest, trees should be fully grown and \textit{not} be pruned, and claimed that random forests do \textit{not} overfit\footnote{ Overfitting describes the situation in which a learner has small ``training" mean-squared error, i.e., performs well in the sample used for estimation, and large ``test" mean-squared error, i.e., performs poorly out-of-sample. }. As \citealp{segal2004machine} pointed out, while fully growing trees indeed reduces the bias, there is a trade-off with variance. He argued that the datasets used by Breiman where inherently difficult to overfit using forests and provided a counter-example in which indeed, contradicting Breiman, random forests overfit if grown deeply. In the same direction, \citealp{lin2006random} showed, by establishing lower bounds on the rate of convergence of the prediction error of nonadaptive\footnote{ Term used in \citealp{lin2006random} to describe a random forest the partition of which is independent of the target variable. } random forests and via simulations for the classical, i.e., adaptive, forests,  that growing large trees does not always give the best out-of-sample performance. In line with these studies, we ask the following question: \textit{how does (su)bagging perform, compared to optimally - in terms of bias and variance - grown trees, if its subtrees\footnote{ In the present paper, we use the term \textit{subtree} to describe a tree grown on a subsample.}, either are, or are not, also optimally grown?}

The present paper contributes to the literature on tree-based methods in two ways. 

First, we formalize the bias-variance trade-off associated with tree size by establishing \textit{pointwise} consistency of trees under assumptions similar to those of \citealp{breiman1984cart}. To our knowledge, pointwise consistency has not previously been established for CART. Trees are a particular case of local averaging estimators. A general proof of consistency for such estimators dates back to \citealp{stone1977consistent} in the case where weights only depend on the covariates and not on the target variable. For data-driven partitions\footnote{ See \citealp{gyorfi2002distribution}'s chapter on data-dependent partitioning for general consistency results.}, i.e., partitions in which weights depend on both covariates and the target variable, \citealp{breiman1984cart} give sufficient conditions for $L^2$ consistency but do not show that CART satisfies those conditions. The first proof of consistency for CART was given in \citealp{10.1214/15-AOS1321} in the context of random forests with subsampling instead of bootstrapping: they show $L^2$ consistency assuming an additive regression function. \citealp{mentch2016quantifying} showed pointwise asymptotic normality of subagging and random forests, but their limiting normal distributions are centered at the expected prediction and not necessarily the true regression function. \citealp{wager2018estimation} showed pointwise consistency of random forests, but their result is not applicable to trees as they assume that the subsample size is asymptotically negligible compared to the dataset size. Recently \citealp{klusowski2023large} also proved, among other things, $L^2$ consistency for an additive regression function.
In the present paper, in order to guarantee pointwise consistency, we assume in particular that the number of observations in a cell grows at a certain rate with the dataset size. The classical CART is \textit{not} guaranteed to satisfy these assumptions and we give an explanation as to \textit{why} that might be the case: depending on the dataset at hand, CART might not split enough in some regions of the feature space. However, we provide an algorithm\footnote{ See Implementation 1 in Section \ref{subsection_stopping_rules}.} that satisfies our theorem's assumptions while still partitioning based on the CART criterion: we simply do not allow for splits to be performed if they were to give cells with fewer observations than a well-chosen lower bound. In some sense, we \textit{uniformize the number of observations across cells} in order to guarantee pointwise consistency. In our proof, we use honesty\footnote{ Honesty is a concept that allows to get rid of some dependencies to the data when constructing an estimator. See e.g. \citealp{athey2016recursive} for a use of honesty in the context of trees.} to explicitly calculate the bias and variance of a tree estimate \textit{locally}, i.e., at a given value of interest. We show that the bias depends on the diameter of the partition's cell containing this value of interest, and that the variance depends on the number of observations inside the cell. This allows us to formalize the bias-variance trade-off associated with tree size. In simulations, we illustrate this trade-off by comparing ``small", ``large" and ``consistent" trees. Moreover, we point out that \textit{consistency implies stability}\footnote{ Here stability is defined as in \citealp{buhlmann2002analyzing}, i.e., locally and asymptotically.}, which means that \textit{trees can be stable}, if appropriately grown. This goes against the commonly stated view that ``randomization works \textit{because} trees are unstable"\footnote{ For example, \citealp{soloff2023bagging} have established ``algorithmic" stability of bagging, a finite sample definition that formalizes the heuristic definition originally given by Breiman. They show in simulations that subagging has a highly stabilizing effect on regression trees, concluding that trees are very unstable, but their trees are grown large in the first place: they use a depth of 50 for a dataset of 500 observations.}. We show in simulations that subagging can improve consistent trees by variance reduction. In other words, we find that \textit{subagging can also improve upon stable learners}.

Second, we study in simulations how subagging\footnote{ Subagging has previously been used instead of bagging in the context of trees in e.g. \citealp{10.1214/15-AOS1321}, \citealp{wager2018estimation} and \citealp{klusowski2023large}.} performs for trees of different sizes. We measure performance in terms of mean-squared error. We fix the same number of splits among a single tree and every randomized tree constituting subagging when comparing the performance of the two methods. This can be obtained using, e.g., the \textbf{randomForest} package in \Rlogo\, but we implemented our own versions of CART and subagging\footnote{ See Implementation 2 in Section \ref{subsection_stopping_rules}, and Section \ref{subsection_R_implem} for replication details.} in order to extract additional information. Starting with stumps, we look at the effect of subagging on the weight of observations around the value of interest. We illustrate that \textit{subagging assigns positive weight to observations that had zero weight in the tree}, conditionally on the dataset. To our knowledge, we are the first to show this explicitly for CART. The idea however that bagging stabilizes predictions by \textit{balancing} the importance of observations in the data goes back to \citealp{grandvalet2004bagging}. They argued that bagging can improve or deteriorate the base learner depending on the quality of influential observations in the context of point estimation and regression, but not for trees. In \citealp{grandvalet2006stability} classification trees are considered, but the effect of bagging on weights is not illustrated as in the present paper. Then, we find that \textit{subagging reduces the variance of a tree around split points}, which is in line with \citealp{buhlmann2002analyzing}'s ``instability regions"\footnote{ Defined as a neighbourhoods around the split points. See Section \ref{subsection_cond_X}.}. The more we split, the larger the instability region, and \textit{improvement is larger for trees with many splits than it is for small trees}. However, \textit{a single tree grown at optimal size can outperform subagging with large trees}. Therefore, subagging large trees is not always a good idea. Nonetheless, subagging can still improve upon a single tree if both methods are optimally grown, and hence to be preferred in practice. 
Additionally, both for the single tree and subagging, our simulations suggest that there is a \textit{linear relation} between the number of observations in the data and the optimal number of splits. In practice, this means that one can first find the optimal size for a single tree in the usually way, e.g., with cross-validation, and then deduce the right size for the ensemble method at hand.

\textbf{Related work.}
We are not the first to consider the importance of size in tree-based methods. The closest existing study that we are aware of is \citealp{zhou2023trees}\footnote{ Our respective starting problematics partially overlap, as both build on the observations made by \citealp{segal2004machine} and \citealp{lin2006random}.}. They argue that tree depth\footnote{Note that, in their simulations, size, rather than depth, is considered, ``as a proxy for depth".} has a regularizing effect on randomized ensembles. They compare the performance of random forests versus bagging as a function of size 
for different signal-to-noise ratios and find that small trees are advantageous when the ratio is low while larger trees should be used when the ratio is high. \citealp{duroux2018impact} show that the global mean-squared error of fully grown quantile forests\footnote{ A special type of forests in which splits only depend on the covariates but not on the target.} with subsampling and that of small randomized trees without subsampling have the same bounds. Both studies always keep the randomization of feature selection in their simulations, while we compare subagging with deterministic, i.e., non-random given the data, trees. Very recently, \citealp{curth2024random} were able to quantify the degree of smoothing of forests, by looking at forests through the lens of nearest neighbours (\citealp{lin2006random}). Moreover they bridge a gap between the notions of bias in statistics and in machine learning (\citealp{dietterich1995machine}) and disentangle multiple distinct mechanisms through which forests improve on trees. In relation to tree size, in their Figure 4, the smoothing effect of ensembling is compared for small and large trees, but trees are also randomized through feature selection. 

\textbf{Organization.}
Section \ref{section_tree_methods} sets the statistical framework and notation, and discusses the considered methods. In particular we give some intuition as to why CART ``avoids" splitting near the edges of the covariate space. Moreover, we heuristically \textit{establish a relation between the bias and variance of subagging to that of a single tree grown on the \textbf{full} sample}\footnote{We do not expect this relation to be difficult to formalize, but leave it for future research.}. Our theorem for pointwise consistency of trees, and an illustration of the bias-variance trade-off associated with tree size, are given in Section \ref{section_consistency}. The effect of subagging on consistent trees is illustrated in Section \ref{subsection_subbag_consistent_trees}. Section \ref{section_subbag_small_trees} examines small trees. We illustrate the effect of subagging on weights conditionally on the data as well as the effect of subagging on stumps conditionally on the covariates. The persistence of smoothing and stabilizing with more splits is discussed. In Section \ref{section_subbag_large_trees} we show that a single tree can outperform subagging if the size of its subtrees is not well chosen. 
The optimal number of splits as a function of the dataset size is also shown. The effect of subsample size and bootstrapping, as well as the replication of our results with readily-available implementations, are finally presented in Section \ref{section_discussion}. Section \ref{conclusion} concludes and proofs are gathered in appendix.

\section{Decision-Tree Methods for Regression} \label{section_tree_methods}
We consider a regression model of the form
\begin{equation}
Y=f(X)+\varepsilon
\label{regression_model}
\end{equation}
with $(X,Y)\in[0,1]^p\times \mathbb{R}$ and $\varepsilon\in\mathbb{R}$ such that $\mathbb{E}[\varepsilon|X]=0$ and $\mathbb{E}[\varepsilon^2|X]=\sigma^2$. Given an random sample $D_n=\{(X_1,Y_1),\dots,(X_n,Y_n)\}$ drawn from (\ref{regression_model}), we want to construct an estimator of $f(x)=\mathbb{E}[Y|X=x]$ based on $D_n$. 

A decision tree partitions $[0,1]^p$ into rectangular cells and, for any $x$, gives as estimator for $f(x)$ the average, noted $T_n(x)$, of all the $Y_i$'s that fall in the same cell as $x$:
\begin{equation}
T_n(x) = \sum_{i=1}^n W_{n,i}(x) Y_i \text{\ \ with \ } W_{n,i}(x) = \frac{\mathbbm{1}_{X_i\in C_n(x)}}{\sum_{j=1}^n\mathbbm{1}_{X_j\in C_n(x)}}
\label{tree_estimator}
\end{equation}
where $C_n(x)$ is the cell that contains $x$. A decision tree, given $D_n$, is fully deterministic. The partition depends on $D_n$ and is obtained by recursive binary splitting of $[0,1]^p$ based on a splitting criterion and a stopping rule. In the case of CART, the splitting criterion is the maximization over splits of the \textit{sample analogue} of
\begin{equation}
\mathbb{V}[Y|X\in C]-\frac{\mathbb{P}(X\in C_L)}{\mathbb{P}(X\in C)}\mathbb{V}[Y|X\in C_L]-\frac{\mathbb{P}(X\in C_R)}{\mathbb{P}(X\in C_R)}\mathbb{V}[Y|X\in C_R]
\label{CART_criterion}
\end{equation}
where $\mathbb{V}$ denotes the variance, $C$ is a cell to be split and $C_L$ and $C_R$ are the two cells that we obtain after splitting $C$. 
CART splits aim at reducing the variation in $Y_i$'s. The stopping rule can be the number of splits, a minimum number of observations in each cell, or a minimum gain in variance reduction.

Bagging is an ensemble method which gives as estimator for $f(x)$ the average, noted $\bar{T}^*_n(x)$, of multiple trees, which are i.i.d. conditionally on $D_n$:
\begin{equation}
\bar{T}^*_n(x) = \frac{1}{B}\sum_{b=1}^B\sum_{i=1}^n W^*_{n,i,b}(x) Y_i
\label{bagging_estimator}
\end{equation}
where $B$ is the number of trees and $W^*_{n,i,b}(x)$ is the random weight for $Y_i$ in the $b^{\text{th}}$ tree. Randomization comes from bootstrapping: for every $b$, we generate a bootstrap sample $D^*_n(b)$ of $D_n$, i.e., a sample of $n$ observations drawn independently and \textit{with} replacement, and grow a tree on $D^*_n(b)$. Defining $W^*_{n,i}(x)=\frac{1}{B}\sum_{b=1}^BW^*_{n,i,b}(x)$, then (\ref{bagging_estimator}) can be re-written as
\begin{equation}
\bar{T}^*_n(x) = \sum_{i=1}^n W^*_{n,i}(x) Y_i. 
\label{bagging_estimator_rewritten}
\end{equation}
Note that all weights sum up to one: $\sum_{i=1}^nW_{n,i}(x)=1$, $\sum_{i=1}^nW^*_{n,i,b}(x)=1$ for all $b$ and $\sum_{i=1}^nW^*_{n,i}(x)=1$. Defining $T^*_{n,b}(x)=\sum_{i=1}^n W^*_{n,i,b}(x) Y_i$, then (\ref{bagging_estimator}) can also be re-written as
\begin{equation}
\bar{T}^*_n(x) = \frac{1}{B}\sum_{b=1}^BT^*_{n,b}(x).
\label{bagging_estimator_rewritten_2}
\end{equation}
In this paper we consider subagging, where bootstrap sampling is replaced by subsampling of size $k$: each replicate $D^*_n(b)$ of $D_n$ on which a subtree is grown is obtained by drawing, \textit{without} replacement, $k\leq n$ observations from $D_n$. 
Several studies, e.g., \citealp{buhlmann2002analyzing}, \citealp{grandvalet2004bagging}, \citealp{buja2006observations}, \citealp{friedman2007bagging}, have suggested that subagging with half the observations has a similar performance to bagging. \citealp{buja2006observations} proved more than that in the case where the base learner is a $U$-statistic (\citealp{hoeffding1992class}), and supported with simulations that the same may hold for trees: bagging with $\alpha_{w} n$ observations is equivalent to subagging with $\alpha_{w/o}n$ observations if $\alpha_{w}=\frac{\alpha_{w/o}}{1-\alpha_{w/o}}$. In such case, subagging has the advantage of being computationally cheaper than bagging (because trees are grown on samples of smaller size). In the present paper we use subagging with half the observations, i.e., we take $k=0.5n$, and show, in Section \ref{subsection_extension_bagging}, that our simulation results are still valid when we vary the subsample size $k$ as well as when we use bagging instead of subagging.

\subsection{CART Criterion and the Location of Splits}\label{subsection_CART_simpler_criterion}

CART, as opposed to other partitioning estimators such as kernel regression (\citealp{nadaraya1964estimating}; \citealp{watson1964smooth}), do \textit{not} split the covariate space uniformly. In order to better understand how CART splits behave, we show the following.  

\begin{proposition} The criterion (\ref{CART_criterion}) can be re-written as
\begin{equation}
\frac{\mathbb{P}(X\in C_L)\mathbb{P}(X\in C_R)}{\mathbb{P}(X\in C)}\{\mathbb{E}[Y|X\in C_L]-\mathbb{E}[Y|X\in C_R]\}^2.
\label{CART_criterion_splified}
\end{equation}
\label{proposition1}
\end{proposition}

\noindent A proof is given in the appendix. 
Expression (\ref{CART_criterion_splified}) is the product of two factors. To understand them, consider the one-dimensional case ($p=1$), assume $X$ is uniformly distributed in $[0,1]$, and let $c$ and $c_l$ be the lengths of intervals $C$ and $C_L$ respectively. The first factor in (\ref{CART_criterion_splified}) is then $\frac{c_l(c-c_l)}{c}$. The length $c$ is fixed when we search for a split inside $C$, and maximizing $c_l(c-c_l)$ over $c_l\in[0,c]$ gives $c_l=\frac{c}{2}$. In other words, the first factor in (\ref{CART_criterion_splified}) is maximized in the middle of $C$, and this is true irrespectively of the distribution of $Y$. The second factor is more complex as it depends on $f$. Note that $\mathbb{E}[Y|X\in C_L]=\mathbb{E}[f(X)|X\in C_L]$ and similarly for $C_R$. If $f$ is constant, then the second factor in (\ref{CART_criterion_splified}) is zero for any split, therefore the entire expression equals zero for any split. If $f$ is linear, then the second factor is constant\footnote{To see that the terms in $s$ cancel each other out, use the identity $(u-v)^2=(u+v)(u-v)$.} for any split and therefore (\ref{CART_criterion_splified}) is maximized in the middle of $C$ (because the first factor is maximized in the middle of $C$). Suppose that $f(x)=x^2$. Then the second factor is maximized where $f$ is steeper, i.e., at the extremity $x=1$. The first factor being maximized in the middle, i.e., at $x=\frac{1}{2}$, this brings the overall argmax at $x=0.64$, which is away from the extremity $x=1$. In other words, \textit{CART tends to split away from the boundaries of a cell}, and, by extension, \textit{CART tends to split away from the boundaries of the feature space}. This supports previous knowledge (\citealp{wager2018estimation}) that CART can be inconsistent at the boundaries of the feature space. Intuitively, not splitting close to the boundaries implies that some cells (precisely those at the boundaries) remain large and hence tree estimates inside those cells will tend to be biased. In order to guarantee pointwise consistency we therefore need somehow to force trees to sometimes split close to the boundary, in order to guarantee that cells become smaller throughout the feature space. \textit{How} we enforce it is given in Section \ref{section_consistency}. 

CART partitions the feature space not based on (\ref{CART_criterion_splified}) but instead based on its empirical analogue, which incorporates the noise variable $\varepsilon$. On the one hand, the larger the noise variance $\sigma^2$, the more a CART split can deviate from its theoretical counter-part obtained by maximizing (\ref{CART_criterion_splified}). On the other hand, the more the observations in a given cell, the closer a CART split will be to its theoretical counter-part.

\subsection{Stopping Rules and Tree Size} \label{subsection_stopping_rules}
A decision tree is obtained by recursive binary splitting of the data. A CART tree cannot be grown further if any of its cells either contains a single observation or yields the same CART criterion for every possible split. A \textit{stopping rule} is a constraint that makes the tree stop growing earlier. When there is no stopping rule, we say that the tree is \textit{unconstrained}. A tree is \textit{fully grown} when each cell contains a single observation. In Proposition \ref{proposition2} we show that, \textit{if unconstrained, CART trees are necessarily fully grown}. More precisely, let $C$ be a cell and let $\nu$ be the number of observations from $D_n$ that are in $C$. Denote $L(s)$ the CART criterion at a split value $s$. 
\begin{proposition}
If $\nu\geq3$, then there exists $s_1$ and $s_2$ such that $L(s_1)\neq L(s_2)$.
\label{proposition2}
\end{proposition}
\noindent In other words, the CART criterion cannot be constant in any cell with at least three observations. A proof is given in the appendix. Because fully grown trees have large variance, stopping rules are needed to guarantee consistency. In this paper we consider two  stopping rules.  First, referred to as \textit{Implementation 1} hereafter, we use bounds on the number of observations in a cell. Second, referred to as \textit{Implementation 2} hereafter, we control for the number of splits. \textit{Bounds on the number of observations indirectly also control for the number of splits}: if we start with a hundred observations and constrain cells to have between ten and twenty observations, then there can only be between five and ten splits. Bounds on the number of observations to some extent also control the location of splits. Indeed, consider the one-dimensional case ($p=1$) and the extreme scenario where we impose a lower bound of fifty observations per cell. Then there is only one admissible split, leaving exactly fifty observations on both sides of the split, precisely between the fiftieth and the fifty-first observations ordered with respect to their $X$ value. Conversely, controlling the number of splits controls the \textit{average}, across cells, number of observations. For example, four splits of a hundred observations will give five cells. Noting the number of observations in each cell by $c_1,\dots,c_5$, then $\sum_{i=1}^5c_i=100$, i.e., $(\sum_{i=1}^5c_i)/5=20$. The greater the number of splits, the smaller the average number of observations. However, controlling the number of splits alone, gives no control over the location of splits: they are chosen based on the CART criterion. 

\textbf{Implementation 1: control the number of observations.} 
We implemented a recursive function that takes as input a dataset $D_n$ and a minimum cell size $h_n$ and returns a partition in which each cell contains at least $h_n$ and at most $2h_n-1$ observations. To do so, at each splitting step, we define admissible splits as those that leave at least $h_n$ observations on each side and we choose among the admissible splits the one that maximizes the CART criterion. On the one hand, this implementation plays an important role in our results on consistency of Section \ref{section_consistency}: we show that consistency is guaranteed if $h_n$ is of the form $n^\alpha$ for some $\alpha>\frac{1}{2}$. By controlling for both the minimum and maximum number of observations allowed in a cell, we are able to guarantee pointwise consistency: \textit{we uniformize number of observations across cells.} On the other hand, this implementation does not always choose the best split in terms of CART criterion, which can happen if such split would yield a cell with fewer than $h_n$ observations. 

\textbf{Implementation 2: control the number of splits.}
Also recursive, takes as input a dataset $D_n$ and a number of splits $N$ and returns a partition with \textit{exactly} $N$ splits. This implementation is used in later sections where we study the performance of trees and subagging as a function of the number of splits. It always finds the best split in terms of CART criterion, but it does not provide exact control over the size of cells obtained from such splits. This implementation replicates existing implementations in \Rlogo\footnote{ See Section \ref{subsection_R_implem} for details and an illustration.}, but we used it in order to be able to extract and illustrate the balancing effect of subagging on weights (see Section \ref{subsection_conditionally_Dn}).

\textbf{Tree size.} We say that a tree is \textit{small}, if either $h_n$ is large or $N$ is small. The smallest possible tree is called a \textit{stump}, for which $h_n=0.5n$ and $N=1$. Conversely, we say that a tree is \textit{large}, if either $h_n$ is small relative to $n$ or if $N$ is large relative to $n$. The largest possible tree is the fully grown tree for which $h_n=1$ and $N=n-1$. In Section \ref{section_consistency} we look at a particular range for $h_n$ that guarantees consistency of trees and in Section \ref{subsection_subbag_consistent_trees} we look at the effect of subagging on consistent trees. The case $N=1$ is detailed in Section \ref{section_subbag_small_trees}. Section \ref{section_subbag_large_trees} shows the performance of subagging compared to trees as a function of $N$.

\subsection{Bias and Variance of Subagging}
Here we heuristically establish a relation between the statistical bias and variance of subagging and that of a single tree grown on the \textit{full} sample. Fix $x_0$ a feature value of interest and let $T_n:=T_n(x_0)$ and $\bar{T}^*_{n,k}:=\bar{T}^*_{n,k}(x_0)$ denote the tree and subagged estimates for $f(x_0)$ as in (\ref{tree_estimator}) and (\ref{bagging_estimator_rewritten_2}) respectively, simplifying in notation the dependency in $x_0$ and making explicit the dependency in the subsample size $k$. Also note $T^*_{n,k,1},\dots,T^*_{n,k,B}$ the subtrees constituting $\bar{T}^*_{n,k}$.
\begin{proposition}
We have
\begin{equation}
\mathbb{E}[\bar{T}^*_{n,k}] = \mathbb{E}[T^*_{n,k,1}]
\label{bias_subbagg}
\end{equation}
and 
\begin{equation}
\mathbb{V}[\bar{T}^*_{n,k}] = \mathbb{V}\mathbb{E}[T^*_{n,k,1}|D_n] + \frac{1}{B}\mathbb{E}\mathbb{V}[T^*_{n,k,1}|D_n]. 
\label{variance_subbagg}
\end{equation}
\label{proposition3}
\end{proposition}
\noindent The proof follows immediately\footnote{ Use that $\mathbb{E}[A]=\mathbb{E}[\mathbb{E}[A|B]]$ and $\mathbb{V}[A]=\mathbb{V}[\mathbb{E}[A|B]]+\mathbb{E}[\mathbb{V}[A|B]]$ for random variables $A$ and $B$.} from the fact that $T^*_{n,k,1},\dots,T^*_{n,k,B}$ are i.i.d. conditionally on $D_n$. 

\textbf{Bias.} From (\ref{bias_subbagg}) we deduce that the bias of subagging with subsample size $k$ is the same as the bias of a single of its subtrees, i.e., a tree grown on a subsample of size $k$. This in turn implies that, \textit{for large $n$, we expect subagging with subsample size $k$ to have similar bias to a single tree grown on the \textbf{full} sample, as long as we use the same number of splits for both methods}. To our knowledge, we are the first to point this out. In later sections we support this statement with simulations. Here we give an informal explanation. Assume that $k$ grows proportionally with $n$, for example, $k=0.5n$. If $n$ is large, then $k$ is large as well. In this case, a tree grown on the entire sample ($n$ observations) and a tree grown on a subsample ($k$ observations) will give splits that are close to each other. Therefore, the cell containing $x_0$ will be of similar diameter for both trees. In Section \ref{section_consistency}, we show that the bias of a tree at $x_0$ depends on the diameter of the cell that contains $x_0$. Therefore, since a tree grown on $n$ observations and a tree grown on $k$ observations give cells of similar diameter for large $n$, we expect indeed both trees to have similar bias. 

\textbf{Variance.} Equation (\ref{variance_subbagg}) is more complex. Again, we proceed with an informal treatment. If the number of subtrees, or subsamples, $B$, is large, then the second term in (\ref{variance_subbagg}) is negligible. The first term, i.e., $\mathbb{V}\mathbb{E}[T^*_{n,k,1}|D_n]$, is the variance of a $U$-statistic. Indeed, 
\begin{equation}
\mathbb{E}[T^*_{n,k,1}|D_n]=\frac{1}{{n \choose k}}\sum_{1\leq i_1<\cdots<i_k\leq n}T_{i_k}
\label{u_statistic_expression}
\end{equation}
where $T_{i_k}:=T_{i_k}(X_{i_1},Y_{i_1},\dots,X_{i_k},Y_{i_k})$ is the prediction for $f(x_0)$ based on a single, deterministic given $D_n$, tree grown on $(X_{i_1},Y_{i_1}),\dots,(X_{i_k},Y_{i_k})$\footnote{ Note that if we would define subagging as the average over \textbf{all} possible subsamples of size $k$, then its variance would be precisely the variance of (\ref{u_statistic_expression}). In fact subagging would \textit{be} a $U$-statistic. See \citealp{mentch2016quantifying} for an in-depth analysis of subagging as $U$-statistics.}. Because (\ref{u_statistic_expression}) is a projection, its variance is smaller than the variance of a single tree based on $(X_1,Y_1),\dots,(X_k,Y_k)$\footnote{ See for example Section 12.1 in \citealp{van2000asymptotic}.}. Following the same reasoning as for the bias, \textit{for large $n$, the variance of the subagged estimator with subsample size $k$ cannot exceed the variance of a single tree estimator grown on the \textbf{full} dataset by more than a factor of $\frac{n}{k}$, provided that the same number of splits is used for both methods.} To our knowledge, we are the first to state this. Indeed, in Section \ref{section_consistency} we show that the variance of a single tree is inversely related to the number of observations used to estimate $f(x_0)$. For large $n$, a tree on the full sample and a tree grown on a subsample will give approximately the same splits, and hence, if $\nu$ is the number of observations in the cell of interest for the full tree, then the subtree will give cells of roughly $\frac{k}{n}\nu$ observations. In particular, if a single tree has close to zero variance for some $x_0$, something that happens away from split points, as we will see Section \ref{section_subbag_small_trees}, then subagging will also have close to zero variance around the same $x_0$.

\section{Pointwise Consistency of Trees and the Bias-Variance Trade-Off Associated with Tree Size} \label{section_consistency} 

In line with other non-parametric methods, such as kernel regression, the main intuition behind consistency still holds for trees: the finer the partition of the feature space, the smaller the bias\footnote{ To be precise, this holds locally where the regression function being estimated is monotone.} of the estimator, while its variance decreases when the number of observations in each cell increases. Therefore, we want trees to be grown deeper when the dataset size $n$ increases, but slowly enough so that the number of observations in each cell increases when $n$ increases. In this section, we consolidate this intuition by showing in Theorem \ref{consistency_theorem} that trees are pointwise consistent when grown under some constraints. In particular, our first implementation described in Section \ref{subsection_stopping_rules} satisfies these constraints. 
We start with some assumptions. 

\begin{assumption}[DGP]
$X\sim\mathcal{U}([0,1]^p)$, $f$ continuous, $\mathbb{E}[\varepsilon|X]=0$, $\mathbb{E}[\varepsilon^2|X]=\sigma^2$.
\label{assumption1}
\end{assumption}
The continuity of $f$ together with the compactness of the feature space allow us to calculate limits of integrals. 
\begin{assumption}[honesty]
For all $n$, for all $i$, $Y_i \indep W_{n,i} | X_1,\dots,X_n$. 
\label{assumption2}
\end{assumption}
In other words, for all $i$, $Y_i$ is independent of $W_{n,i}$ conditionally on $X_1,\dots,X_n$. Honesty
allows us to compute the bias and variance of the tree estimator locally. 
\begin{assumption}[number of observations]
There exists $\frac{1}{2}<\alpha<1$ such that for all $n$ and all $x_0$, noting $h_n=n^\alpha$, almost surely
\begin{equation}
h_n \leq \sum_{i=1}^n \mathbbm{1}_{X_i\in C_n(x_0)} \leq 2h_n-1
\end{equation}
\label{assumption3}
\end{assumption}
where $C_n(x_0)$ again denotes the cell containing $x_0$. This assumption guarantees that the empirical measure of cells tends to zero, and hence, by the Glivenko-Cantelli theorem, the Lebesgue measure of cells will also tend to zero. 
\begin{assumption}[diameter]
For all $x_0$, almost surely $\bigcap_{n\geq1}C_n(x_0)=\{x_0\}$.
\label{assumption4}
\end{assumption}
This additional assumption is imposed to avoid the scenario in which from a certain point onwards, all splits are performed along the same direction. For example, if $p=2$, this would mean obtaining asymptotically a segment in the two-dimensional space, in which case our estimate for $f(x_0)$ would be the average of $f$ over the segment, and hence not necessarily equal to $f(x_0)$. Although Assumption \ref{assumption4} is useful for our proof, the above scenario does not occur in simulations. 
\begin{theorem}
Under Assumptions 1-4, for all $x_0$, $T_n(x_0)\rightarrow f(x_0)$ in probability. 
\label{consistency_theorem}
\end{theorem}
Here we give a sketch of the proof of Theorem \ref{consistency_theorem} in order to illustrate the bias-variance trade-off between cell size and number of observations. A complete proof is given in the appendix. Using honesty, we calculate the squared bias and the variance of the tree estimator at a given $x_0$. We obtain three terms:
\begin{itemize}
\item the squared bias of $T_n(x_0)$, of the form
\begin{equation}
\Bigg(\mathbb{E}\Bigg[\frac{\int_{C_n(x_0)}f(x)dx}{\int_{C_n(x_0)}dx}\Bigg]-f(x_0)\Bigg)^2, 
\label{intuition_bias2}
\end{equation}
\item the variance of the error term $\sum_{i=1}^nW_{n,i}(x_0)\varepsilon_i$, of the form
\begin{equation}
\mathbb{E}\Bigg[\frac{\sigma^2}{\sum_{i=1}^n\mathbbm{1}_{X_i\in C_n(x_0)}}\Bigg], 
\label{intuition_error_term}
\end{equation}
\item and the variance of the regression term $\sum_{i=1}^nW_{n,i}(x_0)f(X_i)$, of the form
\begin{equation}
\mathbb{E}\Bigg[\frac{\int_{C_n(x_0)}f^2(x)dx}{\int_{C_n(x_0)}dx}\Bigg]-\mathbb{E}\Bigg[\frac{\int_{C_n(x_0)}f(x)dx}{\int_{C_n(x_0)}dx}\Bigg]^2. 
\label{intuition_regression_term}
\end{equation}
\end{itemize}
Consistency is obtained when all three terms converge to zero. On the one hand, the diameter of $C_n(x_0)$ needs to tend to zero as $n$ increases, so that the terms (\ref{intuition_bias2}) and (\ref{intuition_regression_term}) tend to zero. Indeed, if that's the case, we use the uniform continuity of $f$ to guarantee for example that $\frac{\int_{C_n(x_0)}f(x)dx}{\int_{C_n(x_0)}dx}$ converges to $f(x_0)$. On the other hand, the number of observations in $C_n(x_0)$ needs to tend to infinity, so that the term (\ref{intuition_error_term}) converges to zero. Therefore, the diameter of $C_n(x_0)$ should not tend to zero too quickly. This summarizes the bias-variance trade-off when choosing how much to grow a regression tree. 

Small trees, as defined in Section \ref{subsection_stopping_rules}, generate partitions of large cells, hence tend to have small variance while remaining biased. Large trees generate finer partitions, and will tend to have small bias but large variance. \textit{Trees that satisfy Assumption \ref{assumption3} are somewhere in between small and large trees: they tend to have a smaller bias compared to small trees and a smaller variance compared to large trees.} We illustrate this with a simulation based on our first implementation of Section \ref{subsection_stopping_rules}. We consider three scenarios:
\begin{enumerate}[label=\alph*)]\centering
\item \textit{small tree:} $h_n=\frac{n}{3}$ which gives a small number of splits for any $n$,
\item \textit{consistent tree:} $h_n=n^{0.65}$, i.e., satisfying Assumption \ref{assumption3}, and
\item \textit{large tree:} $h_n=4$, i.e., cells have between 4 and 7 observations for any $n$.
\end{enumerate}
\textbf{Simulation I.} We take $f(x)=x^2$, $X\sim\mathcal{U}(0,1)$ and $\varepsilon\sim\mathcal{N}(0,0.2^2)$ in (\ref{regression_model}). We generate $200$ datasets of size $2000$, all of which have the same realization of $X_i$'s. For each dataset, and for $n=50,100,150,\dots,2000$, we grow a tree on the first $n$ of the 2000 observations following each of the three scenarios a,b,c). Figure \ref{consistency_figure} shows the empirical mean and the mean $\pm$ one standard deviation conditionally on $X$ of the tree estimate for $f(0.5)=0.5^2$ across replicates for each $n$ and each scenario. The left plot shows scenario a): small trees. The middle plot shows scenario b): consistent trees. The right plot shows scenario c): large trees. These graphs illustrate the bias-variance trade-off: small trees have low variance but are biased while large trees are unbiased but have high variance. Trees satisfying Assumption \ref{assumption3} are in-between: they are unbiased \textit{and} have low variance. 
\begin{figure}[]
\begin{subfigure}{0.325\textwidth}
	\includegraphics[width=\textwidth]{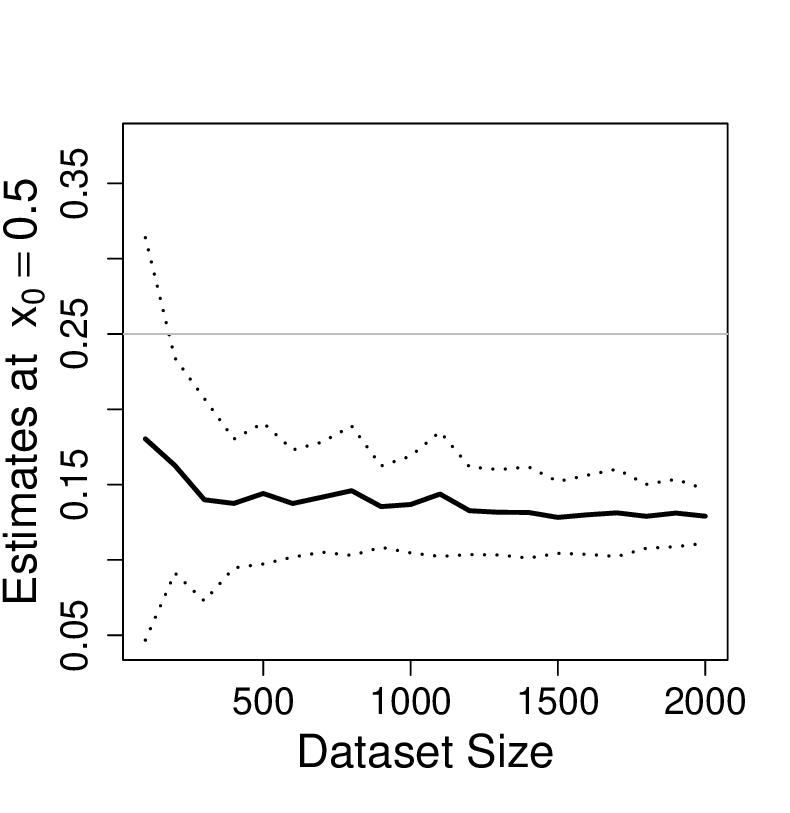}
	\caption{Small Tree}
\end{subfigure}
\begin{subfigure}{0.325\textwidth}
	\includegraphics[width=\textwidth]{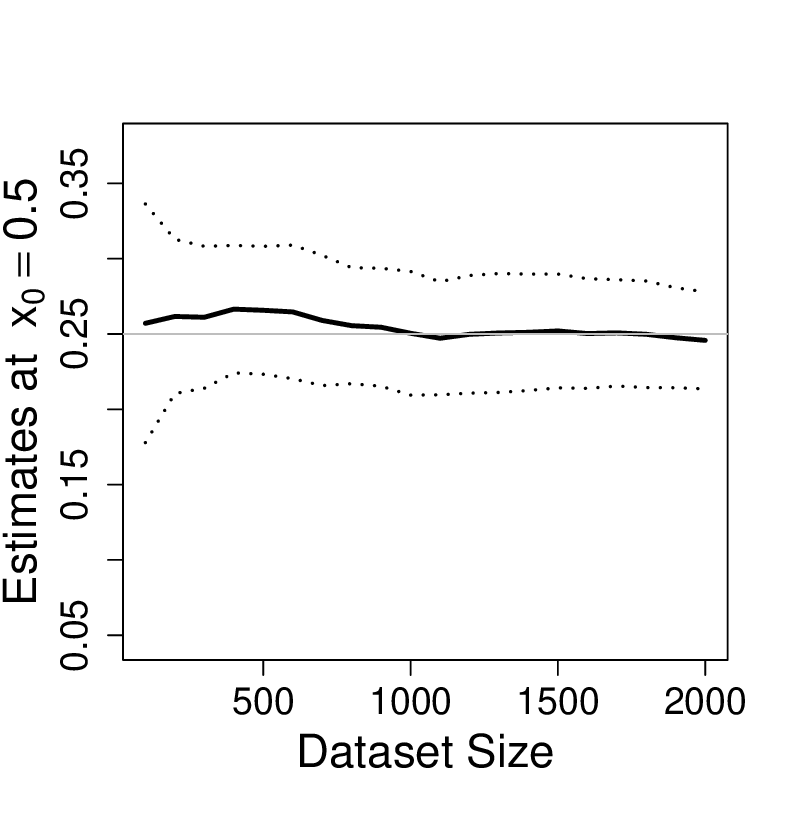}
	\caption{Consistent Tree}
\end{subfigure}
\begin{subfigure}{0.325\textwidth}
	\includegraphics[width=\textwidth]{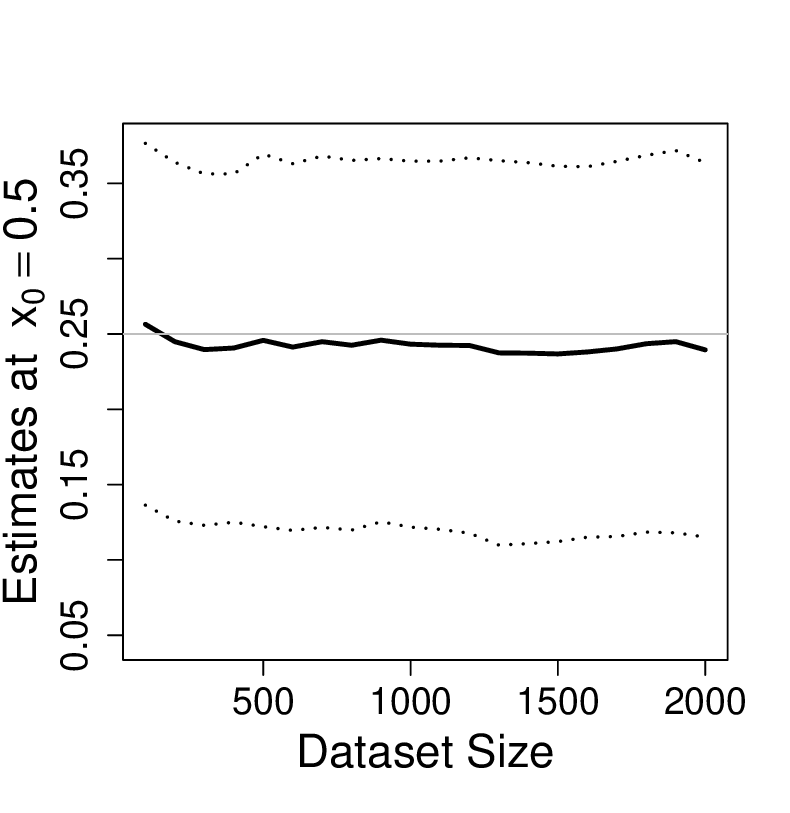}
	\caption{Large Tree}
\end{subfigure}
\caption{(In)consistency of trees: on the $x$-axis is $n$; the solid (resp. dotted) line represents the sample mean (resp. mean $\pm$ one standard deviation) of the tree estimate for $f(x_0)$ (grey line) in each scenario a), b) and c).}
\label{consistency_figure}
\end{figure}

Figure \ref{consistency_figure_MSE} shows the corresponding biases, variances, and mean-squared errors. The mean squared error at $x_0=0.5$ converges to zero for the consistently grown tree while it does not converge to zero for the small and large trees. In the case of the small tree, the bias does not vanish. For the large tree, the variance does not vanish. 
\begin{figure}[]
\begin{subfigure}{0.325\textwidth}
	\includegraphics[width=\textwidth]{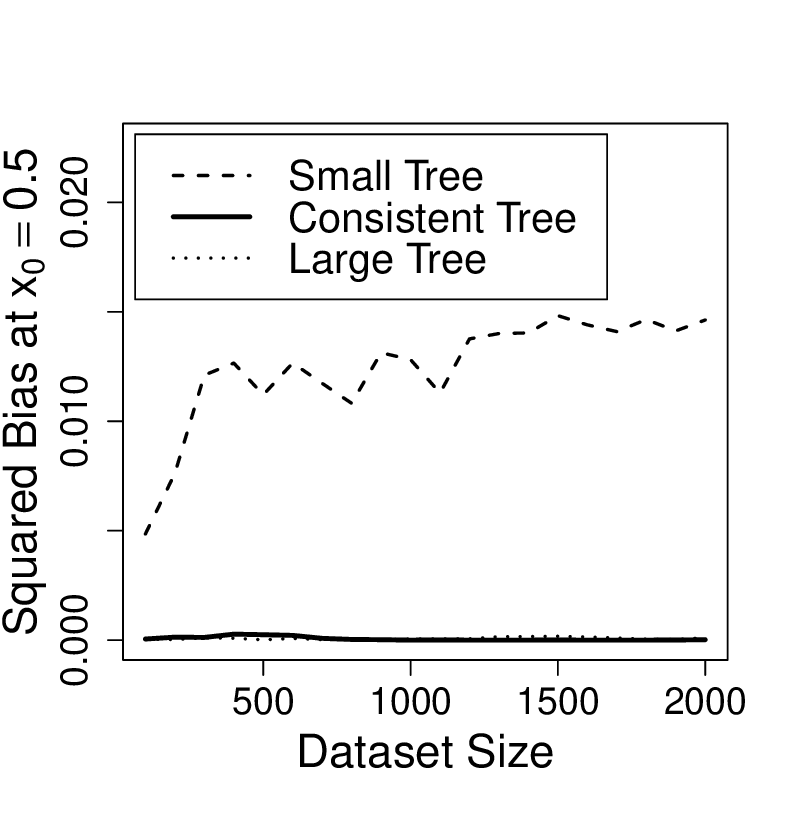}
\end{subfigure}
\begin{subfigure}{0.325\textwidth}
	\includegraphics[width=\textwidth]{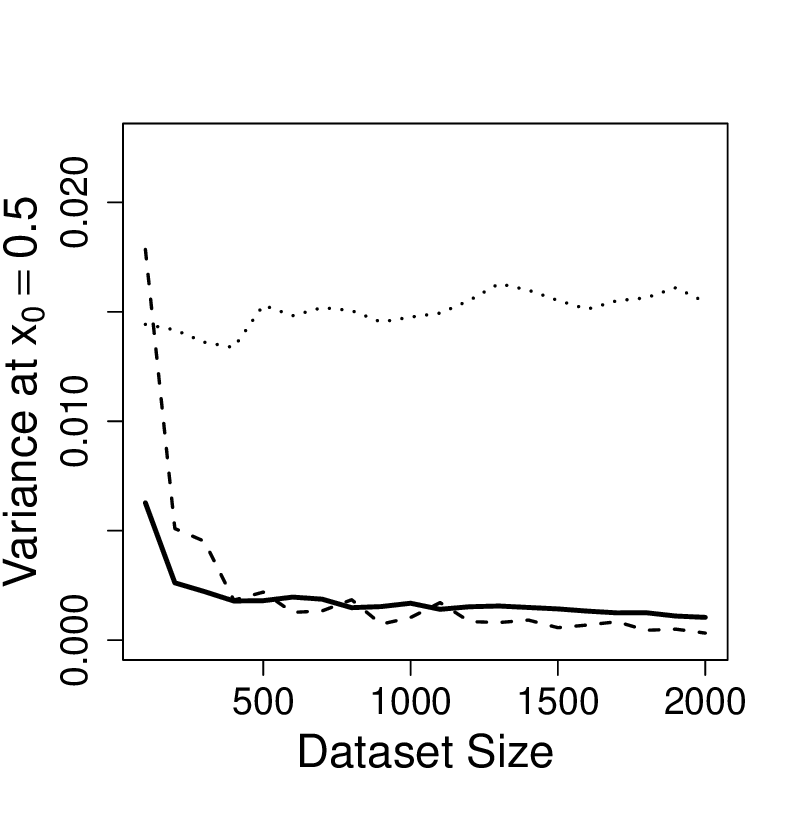}
\end{subfigure}
\begin{subfigure}{0.325\textwidth}
	\includegraphics[width=\textwidth]{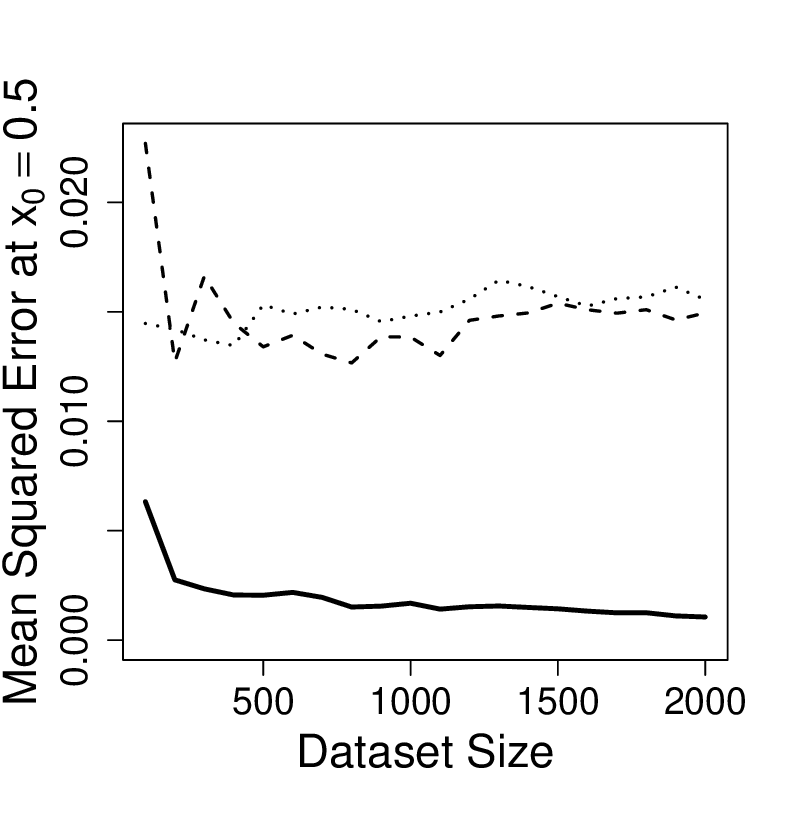}
\end{subfigure}
\caption{Bias-variance trade-off associated with tree size. }
\label{consistency_figure_MSE}
\end{figure}

\section{Subagging Consistent Trees} \label{subsection_subbag_consistent_trees}

We show via simulation that \textit{subagging consistent - and hence stable - trees does not affect the bias while it can improve the variance}. We saw in Theorem \ref{consistency_theorem} that a tree grown on the full sample is consistent when the minimum cell size $h_n$ grows appropriately with $n$, i.e., is of the form $n^\alpha$ for some $\alpha>\frac{1}{2}$. To guarantee consistency of \textit{each} subtree, we similarly choose $h_k$ of the form $k^\alpha$, again for some $\alpha>\frac{1}{2}$. Then the subagged estimator, which is an average of such subtrees, will also be consistent. 

Keeping the same $\alpha$ for the tree and the subagged estimator implies that cells will be of similar diameter for both estimators and hence, based on (\ref{intuition_bias2}), we expect the two estimators to have similar bias. Each subtree will have cells of fewer observations than the original tree and therefore, based on (\ref{intuition_error_term}), each subtree is expected to have a larger variance than the original tree. However the subagged estimator is an average of subtrees, which means that ultimately the average is taken over more observations compared to the original tree, and hence we expect a reduction in variance. The effect on variance is further examined in Section \ref{section_subbag_small_trees}. 

Figure \ref{sub_consistency_figure} shows the effect of subagging on consistent trees in the same simulation as in Figure \ref{consistency_figure} (bottom plot) for scenario b), i.e., $h_n=n^{0.65}$. Here the subagged estimator is defined as the average of $B=50$ trees\footnote{ Extensive empirical evidence exists in the literature that the number of trees need not be large in general. In our simulations, increasing $B$ to 100 or 150 did not bring any change to the results.}, each of which is grown on a subsample of size $k=0.5n$ and with a minimum cell size of $h_k=k^{0.65}$. We observe indeed that, in terms of bias, the two estimators behave similarly, while subagging reduces variance. 
\begin{figure}[]
\begin{subfigure}{0.325\textwidth}
	\includegraphics[width=\textwidth]{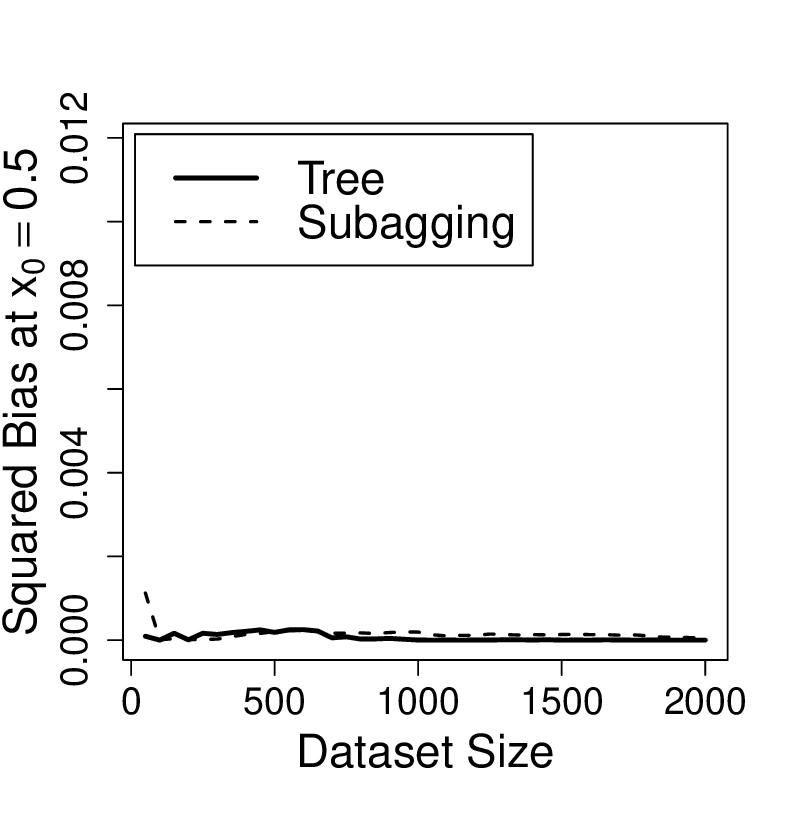}
\end{subfigure}
\begin{subfigure}{0.325\textwidth}
	\includegraphics[width=\textwidth]{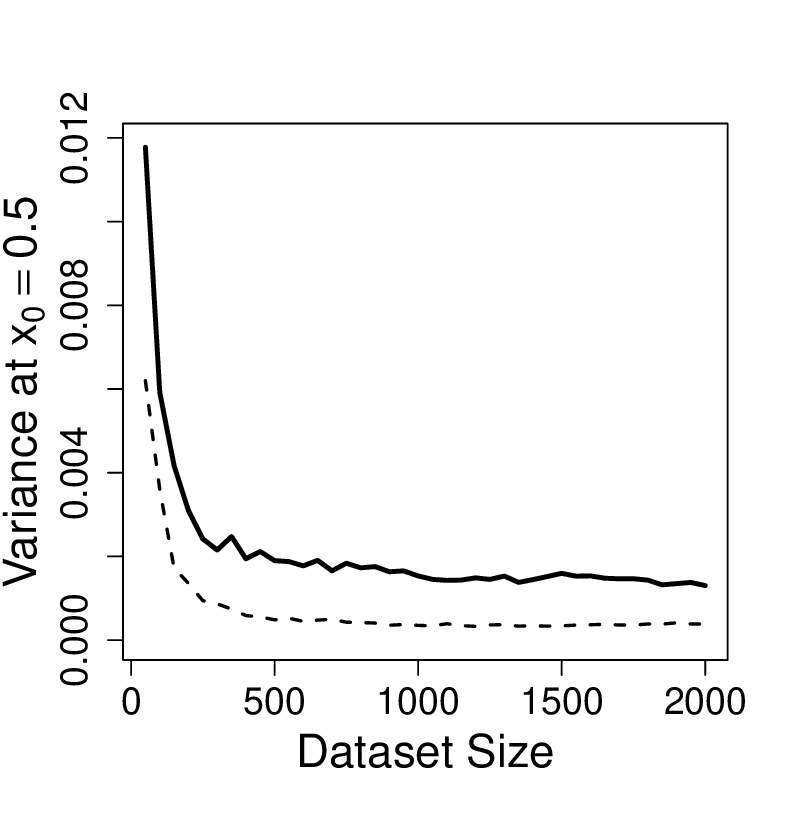}
\end{subfigure}
\begin{subfigure}{0.325\textwidth}
	\includegraphics[width=\textwidth]{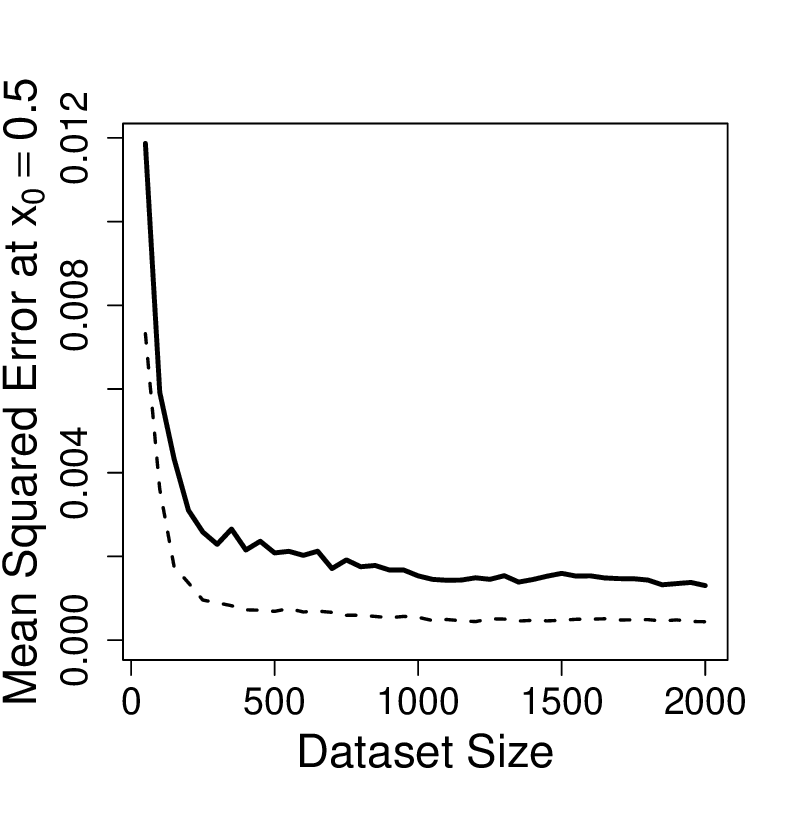}
\end{subfigure}
\caption{Subagging consistent trees: on the $x$-axis is $n$; the solid (respectively dashed) line represents the squared bias (left plot), variance (middle) and mean squared error (right) of the tree (respectively subagged tree) estimates for $f(x_0)$ when consistently grown ($\alpha=0.65$).}
\label{sub_consistency_figure}
\end{figure}

\section{Subagging Small Trees} \label{section_subbag_small_trees}

In this section we look at stumps, i.e., single-split trees, as proxy for small trees, because this allows us to explain the effect of subagging on weights. To do so, we start with an analysis conditionally on $D_n$. We show that \textit{conditionally on $D_n$, subagging increases the number of distinct observations used to estimate $f(x_0)$ compared to a single tree, and these observations cover a wider part of the feature space}. Additionally, \textit{the closer $x_0$ is to the split point, the more subagging adds weight to observations that had zero weight in the single tree}. Then we look at the tree and subagged estimates conditionally on $X$. We support existing knowledge, e.g., \citealp{buhlmann2002analyzing}, that \textit{aggregating improves upon single trees by reducing the variance around the split point, region in which a stump has high variance}.

\subsection{Analysis Conditionally on $D_n$} \label{subsection_conditionally_Dn}
Fix $D_n$ and $x_0$ and take $p=1$. Given $D_n$, the first split of a decision tree is fixed and partitions $[0,1]$ in two regions, one of which contains $x_0$. Let $s\in[0,1]$ be the split point and assume, without loss of generality, that $x_0 < s$. Then the tree estimate for $f(x_0)$, noted $T_n(x_0)$, is the average of all observations in $D_n$ such that $X_i\leq s$. Now let $D^*_n(b)$ be a subsample of $D_n$ and let $s^*$ be the first split point of a tree grown on $D^*_n(b)$. There are three possible scenarios: 
\begin{enumerate}[label=\Alph*)]\centering
\item $x_0<s^*<s$,
\item $x_0<s<s^*$, 
\item $s^*<x_0<s$.
\end{enumerate}
First, assume that $x_0<s^*$. Then the subtree estimate for $f(x_0)$ based on $D^*_n(b)$, noted $T^*_{n,b}(x_0)$, is the average of all observations in $D_n$ such that $X_i\leq s^*$ and $X_i\in D^*_n(b)$. In case $x_0<s^*<s$, then $T^*_{n,b}(x_0)$ is also an average of observations that are in $T_n(x_0)$ but fewer of them. When $x_0<s<s^*$, then $T^*_{n,b}(x_0)$ is an average of observations that are in $T_n(x_0)$ and some observations that are not in $T_n(x_0)$ since they are such that $s<X_i$. The same holds for the case where $s^*<x_0<s$. The subagged estimator being an average of several subtrees, if we take enough subsamples, $\bar{T}^*_n(x_0)$ will be an average of observations such that $X_i<s$ and observations such that $X_i>s$. In other words, conditionally on $D_n$, subagging increases the number of distinct observations used to estimate $f(x_0)$ compared to a single tree, and these observations cover a wider part of $[0,1]$. In order to illustrate the effect of subagging on weights, we consider the following simulation with $n$ fixed.

\textbf{Simulation II.} We use our second implementation as described in Section \ref{subsection_stopping_rules}. We generate \textit{one} realization of $n=100$ observations from the same DGP (\ref{regression_model}) as in Simulation I, i.e., with $f(x)=x^2$, $X_i$ i.i.d. $\mathcal{U}[0,1]$ and $\varepsilon_i$ i.i.d. $\mathcal{N}(0,0.2^2)$. The first (left) plot in Figure \ref{effect_weights} shows the estimate obtained from a stump (in blue) and a subagged stump (in red). We use subsamples of size $k=0.5n$ and average over $B=50$ randomized trees to get the subagged estimates. The optimal CART split\footnote{ I.e., obtained by maximization of the \textit{empirical} analogue of (\ref{CART_criterion}), or equivalently, of (\ref{CART_criterion_splified}).} is at $s=0.63$.\footnote{ Note the small difference between 0.63, which is obtained based on the data, and the value of 0.64 previously obtained in Section \ref{subsection_CART_simpler_criterion} based on (\ref{CART_criterion_splified}).} Given $D_n$ and $x_0$, the tree weights $\{W_{n,i}(x_0)\}_{i=1,\dots,n}$ are deterministic. They are constant and equal to $\frac{1}{\sum_{j=1}^n\mathbbm{1}(X_j<s)}$ for all $Y_i$ such that $X_i<s$ and equal to zero elsewhere. For a subsample $b$, the weights $\{W^*_{n,i,b}(x_0)\}_{i=1,\dots,n}$ are random conditionally on $D_n$. They are constant for \textit{some} (because of subsampling) $Y_i$'s such that $X_i<s^*$, and zero elsewhere. The average over subsamples, i.e. the weights $W^*_{n,i}(x_0)$, are also random conditionally on $D_n$, and are as follows. 

\textbf{The case where $x_0$ is far from $s$.} The second plot in Figure \ref{effect_weights} shows the weights associated with $x_0=0.1$, which is represented by the first (from left to right) perpendicular line in the first plot. 
\begin{enumerate}[label=(\roman*)]
\item For some subsamples, we will have $s^*<s$ (scenario A) and hence $W^*_{n,i,b}(x_0)$ will be zero for those observations that satisfy $s^*<X_i<s$, while for those same observations, $W_{n,i}(x_0)>0$. Thus, we expect $W^*_{n,i}(x_0)$ to be smaller than $W_{n,i}(x_0)$ close to and to the left of $s$. We indeed observe this in the second plot for observations in the region $[0.5,s]$.
\item For other subsamples, we will have $s<s^*$ (scenario B) and hence $W^*_{n,i,b}(x_0)$ will be non-zero for those observations that satisfy $s<X_i<s^*$ and $X_i\in D^*_n(b)$, while for those same observations, $W_{n,i}(x_0)=0$. Thus, we expect $W^*_{n,i}(x_0)$ to be larger than $W_{n,i}(x_0)$ (and in particular non-zero) close to and to the right of $s$. This is observed in the second plot in the region $[s,0.6]$.
\item $s^*$ will, in general, for subsamples of enough observations, not fall very far from $s$. Thus observations for which $X_i$ far to the left of $s$, will also satisfy $X_i<s^*$, and hence $W^*_{n,i}(x_0)$ will be close to $W_{n,i}(x_0)$ (close and not exact, because of subsampling). This is observed in the second plot in the region $[0,0.5]$.
\item For the same reason, observations for which $X_i$ is far to the right of $s$, will also satisfy $s^*<X_i$ and hence will have zero weight in the subagged estimate ($W^*_{n,i}(x_0)=0$). This is observed in the second plot in the region $[0.6,1]$.
\end{enumerate}

\textbf{The case where $x_0$ gets closer to $s$.} The closer $x_0$ is to $s$, the more likely it is that for some subsamples, we will have $s^*<x_0$. For such subsamples, every observation to the left of $s^*$ will be excluded from the estimate (therefore, zero weight). Consequently, $W^*_{ni}(x_0)$ will tend to be smaller than $W_{ni}(x_0)$. At the same time, observations to the right of $s^*$ will be included, which gives the non-zero weights $W^*_{ni}(x_0)$ for all observations to the right of $s$. This is shown in the third and fourth plots of Figure \ref{effect_weights} for $x_0=0.5$ and $x_0=0.6$, corresponding to the second and third respectively perpendicular lines in the first plot. In other words, the closer $x_0$ is to the split point, the more subagging adds weight to observations that had zero weight in the single tree. To our knowledge, we are the first to make this observation. This also helps understand the variance reduction of subagging observed around the split points, illustrated next in Section \ref{subsection_cond_X}, as \textit{estimates obtained from subagging are based on more distinct observations than in a single tree}. 
\begin{figure}[]
\begin{subfigure}{0.24\textwidth}
	\includegraphics[width=\textwidth]{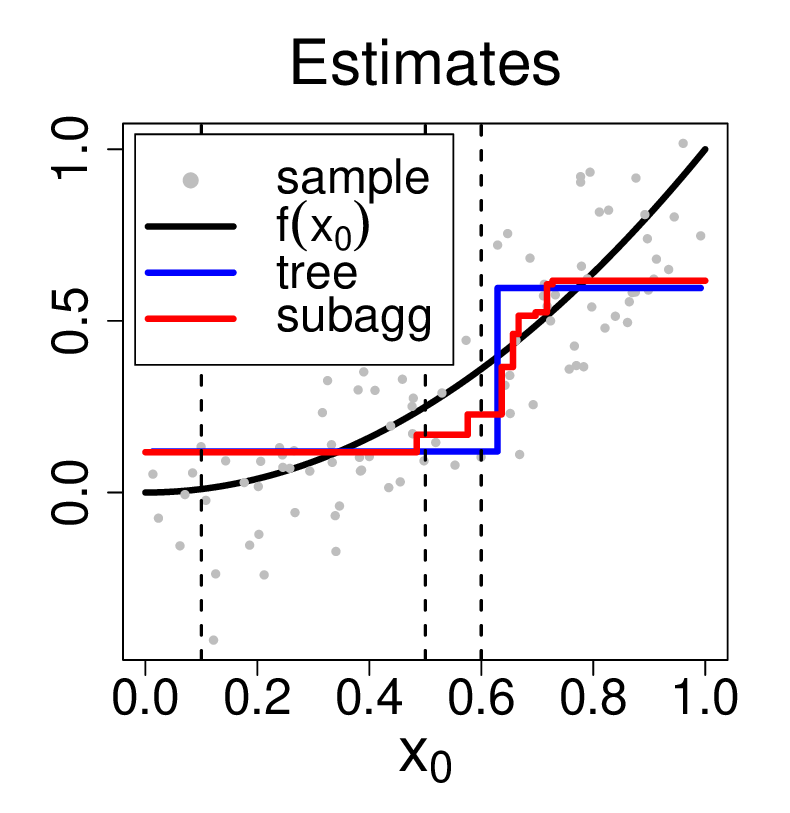}
\end{subfigure}
\begin{subfigure}{0.24\textwidth}
	\includegraphics[width=\textwidth]{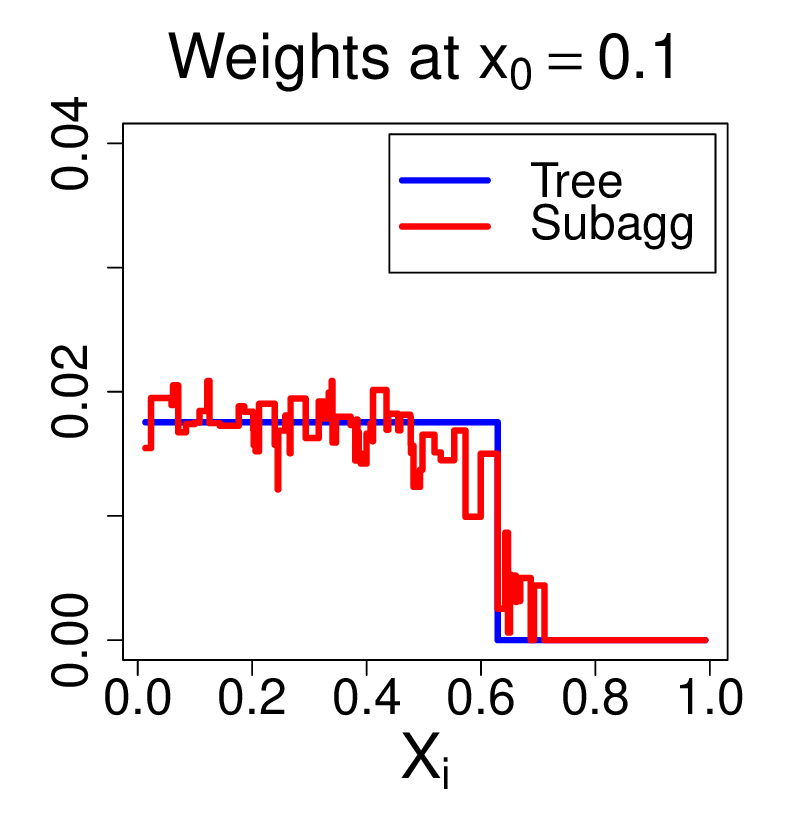}
\end{subfigure}
\begin{subfigure}{0.24\textwidth}
	\includegraphics[width=\textwidth]{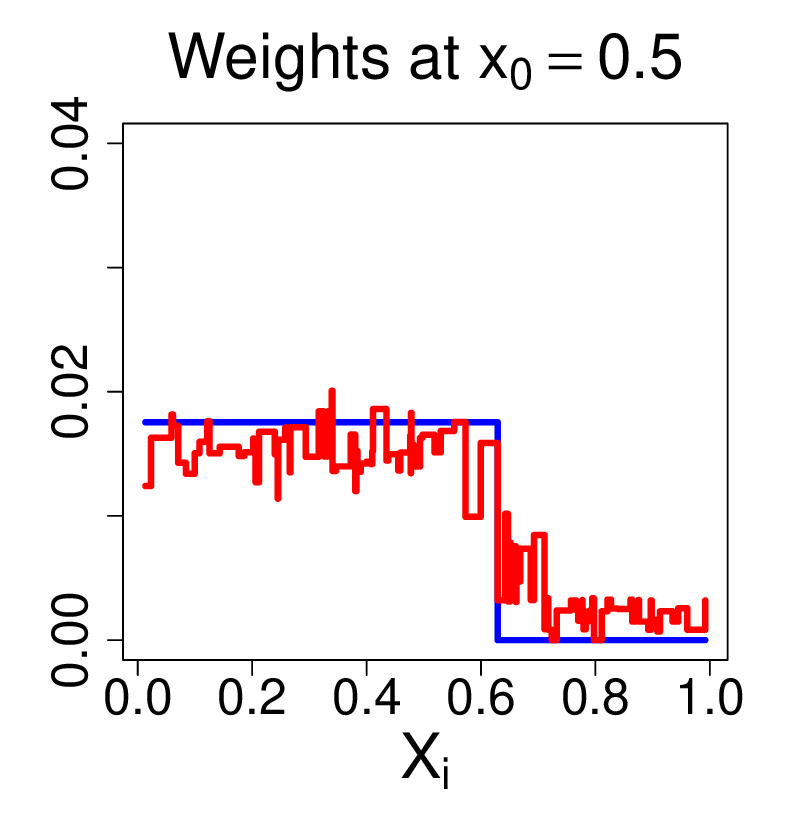}
\end{subfigure}
\begin{subfigure}{0.24\textwidth}
	\includegraphics[width=\textwidth]{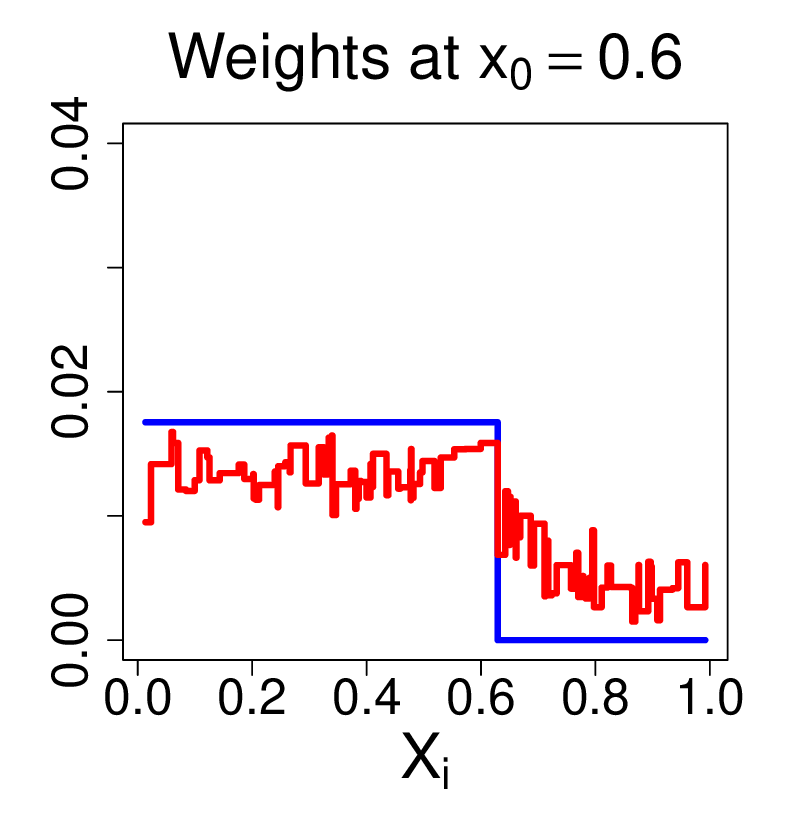}
\end{subfigure}
\caption{\textit{First plot (left)}: stump (blue) and subagged stump (red) estimates as a function of $x_0$ for one realization of $D_n$ (gray points). In black is the true regression function $f(x_0)$. \textit{Second plot}: weights $W_{n,i}(x_0)$ (stump, blue) and $W^*_{n,i}(x_0)$ (subagged stump, red) for $x_0=0.1$. \textit{Third plot:} weights for $x_0=0.5$. \textit{Fourth plot:} weights for $x_0=0.6$.}
\label{effect_weights}
\end{figure}

\subsection{Statistical Bias and Variance}\label{subsection_cond_X}
In Section \ref{subsection_conditionally_Dn} we looked at a single realization of $D_n$. Here we are interested in the effect of subagging on the statistical bias and variance of stumps, hence we look at several realizations of $D_n$, keeping the $X_i$'s fixed. Even though the $X_i$'s are fixed, different realizations of the $Y_i$'s will perturb the split point. We show that \textit{subagging has a \textbf{smoothing} effect which reduces the bias compared to a stump but this effect \textbf{disappears} when the number of splits increases}. Additionally, \textit{subagging has a \textbf{stabilizing} effect which reduces the variance compared to a tree and this effect \textbf{persists} when the number of splits increases}.

Figure \ref{sub_bagging_stumps} shows the squared bias and variance conditionally on $X$ of a stump versus a subagged stump obtained over $200$ replicates of $D_n$ generated from Simulation II, as in Figure \ref{effect_weights}. 
\begin{figure}[h]
\begin{subfigure}{0.325\textwidth}
	\includegraphics[width=\textwidth]{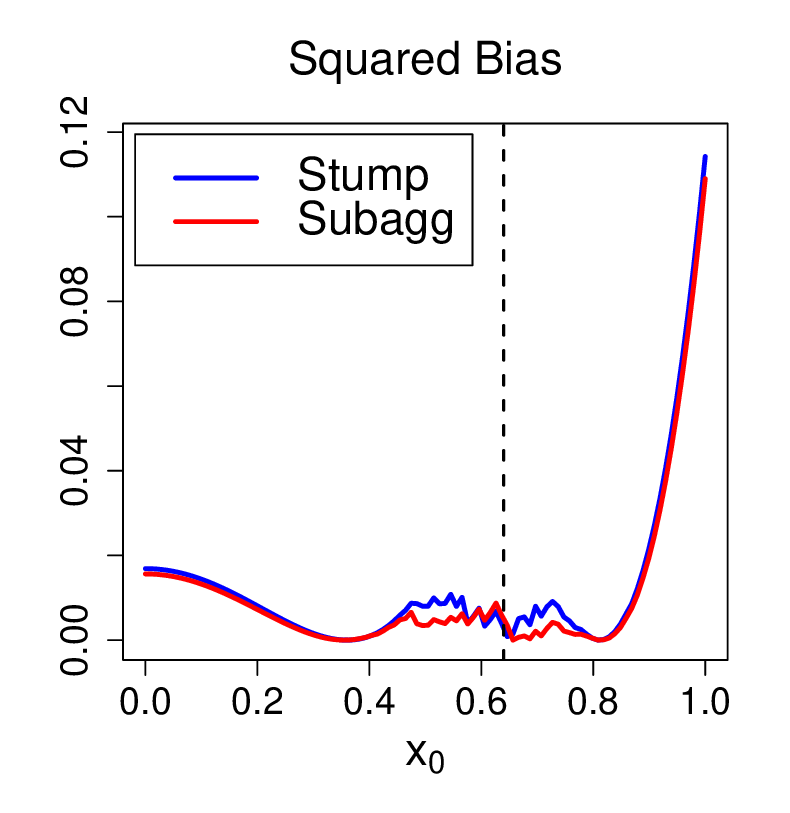}
\end{subfigure}
\begin{subfigure}{0.325\textwidth}
	\includegraphics[width=\textwidth]{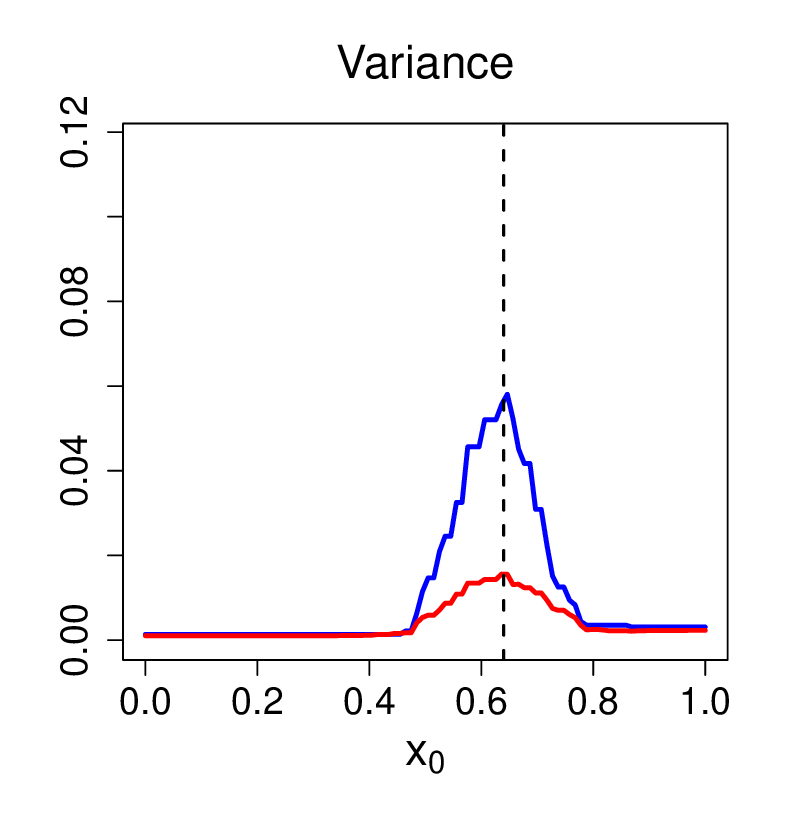}
\end{subfigure}
\begin{subfigure}{0.325\textwidth}
	\includegraphics[width=\textwidth]{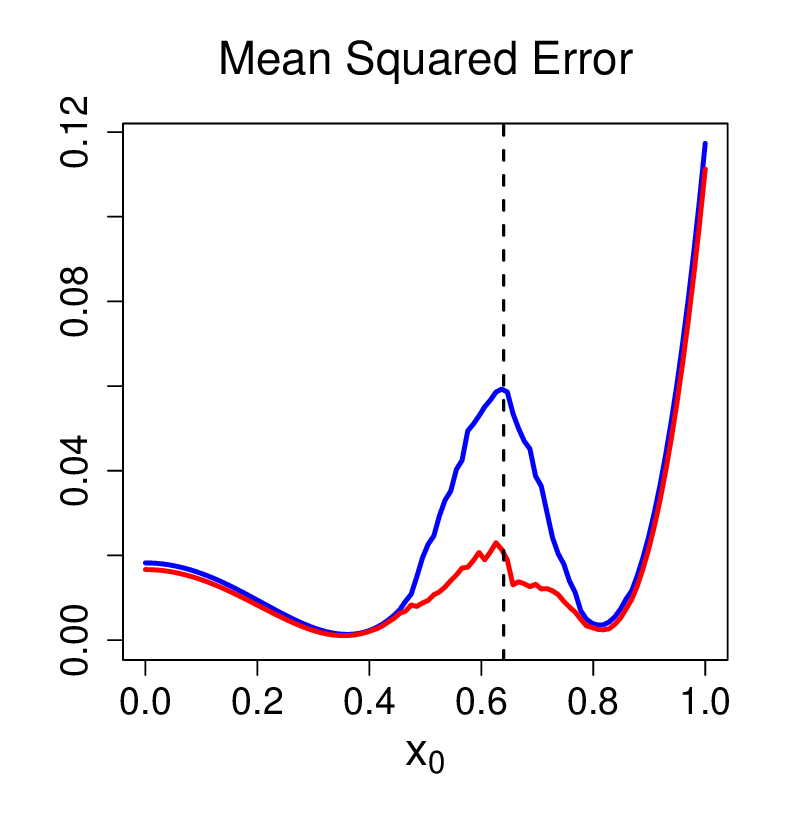}
\end{subfigure}
\caption{Squared bias, variance and mean squared error of a stump (blue) and subagged stump (red) as a function of $x_0$. The vertical line shows the theoretical split ($x=0.64$).}
\label{sub_bagging_stumps}
\end{figure}
Figure \ref{sub_bagging_small_trees} shows trees with $N=3$ splits. As a reference, the first two theoretically optimal CART splits, i.e., maximizing (\ref{CART_criterion}), are at $x=0.64$ and $x=0.83$ respectively. The smoothing effect is visible around $x=0.64$ for a single stump but no longer in the case $N=3$. The variance reduction is visible in both the stump and in the case $N=3$. In particular, \citealp{buhlmann2002analyzing}'s ``instability region" is well illustrated: they argued that this region is a $n^{-\frac{1}{3}}$-neighborhood of the best split. Here $n=100$, i.e., $n^{-\frac{1}{3}}\approx0.22$. While it is easy to see that subagging improves upon a tree in terms of global mean-squared error in both the case of a stump and the case $N=3$, it is less clear in which case (between $N=1$ and $N=3$) subagging improves, overall in $x_0$, \textit{the most}. This is further discussed in Section \ref{section_subbag_large_trees}. 
\begin{figure}[]
\begin{subfigure}{0.325\textwidth}
	\includegraphics[width=\textwidth]{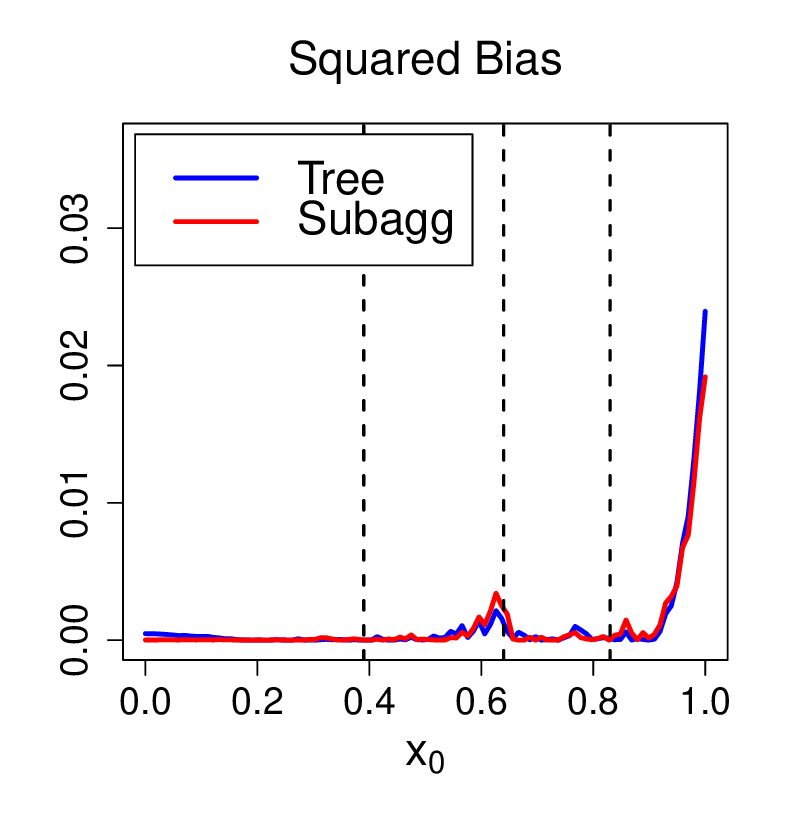}
\end{subfigure}
\begin{subfigure}{0.325\textwidth}
	\includegraphics[width=\textwidth]{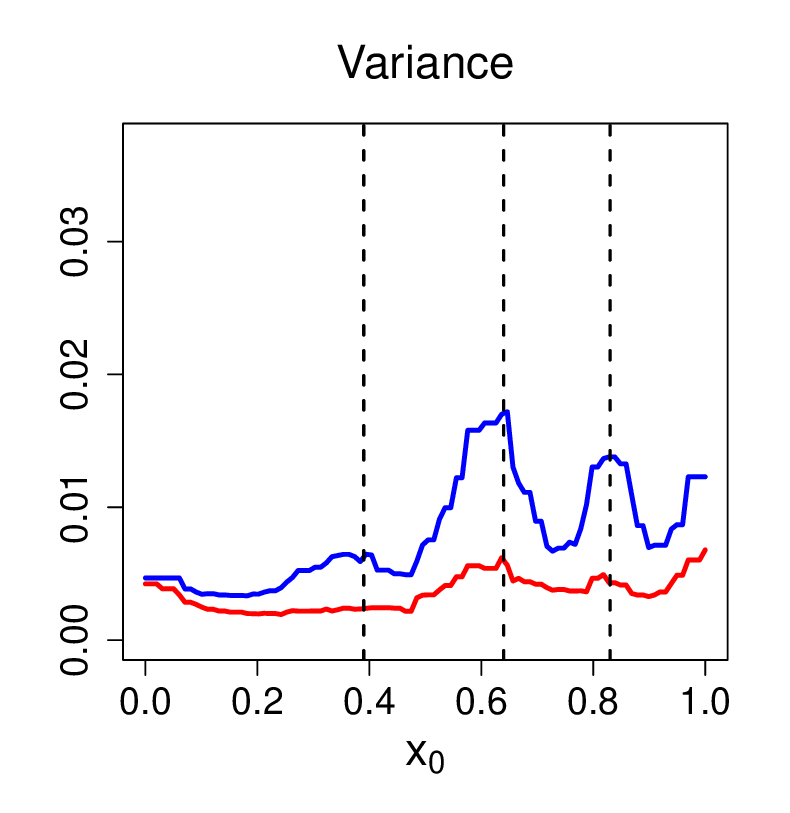}
\end{subfigure}
\begin{subfigure}{0.325\textwidth}
	\includegraphics[width=\textwidth]{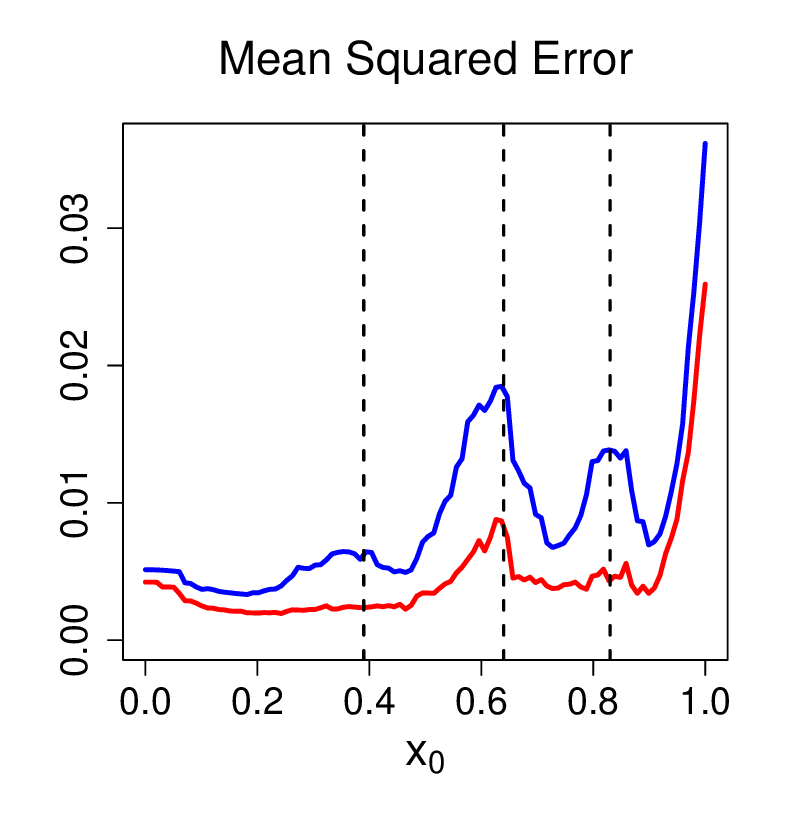}
\end{subfigure}
\caption{Three-split tree (blue) versus subagging of three-split trees (red). }
\label{sub_bagging_small_trees}
\end{figure}

\section{Subagging Large Trees} \label{section_subbag_large_trees}

Now we show that \textit{subagging large trees is not always a good idea}. As seen in Section \ref{subsection_cond_X}, for stumps, subagging has a small effect on bias around the split point and such effect disappears with a few more splits: we expect that \textit{subagging brings no improvement on bias when the number of splits increases further, and in particular if it is large}. Moreover, subagging reduces the variance around the split points, and this effect persists when trees are grown deeper: \textit{the more we split, the more instability we create, and hence the larger the region over which subagging helps by reducing the variance}\footnote{ This is in line with \citealp{buhlmann2002analyzing}.}. Figure \ref{sub_bagging_large_trees} shows the global performance, i.e., on average over $x_0$, of trees versus subagging as a function of the number of splits in the same simulation framework as in Figures \ref{sub_bagging_stumps} and \ref{sub_bagging_small_trees}. 

\textit{Large trees become very unstable, and even though subagging significantly reduces the variance of such trees, subagging performs worse than a single tree appropriately grown}. In Figure \ref{sub_bagging_large_trees}, obtained again from Simulation II, the optimal (in terms of mean-squared error) number of splits for a single tree is $N=4$, while for subagging it is $N=3$. These are represented by the two dotted vertical lines in the right plot of Figure \ref{sub_bagging_large_trees}. The corresponding mean-squared errors are 0.87\% (horizontal grey line) and 0.44\% respectively, meaning an improvement of factor 2. In comparison, the mean-squared error of a tree with $N=49$ splits is 3.74\% and for its subagged analogue 1.48\%, i.e., almost \textit{twice} the mean-squared error of a single tree optimally grown.
\begin{figure}[]
\begin{subfigure}{0.325\textwidth}
	\includegraphics[width=\textwidth]{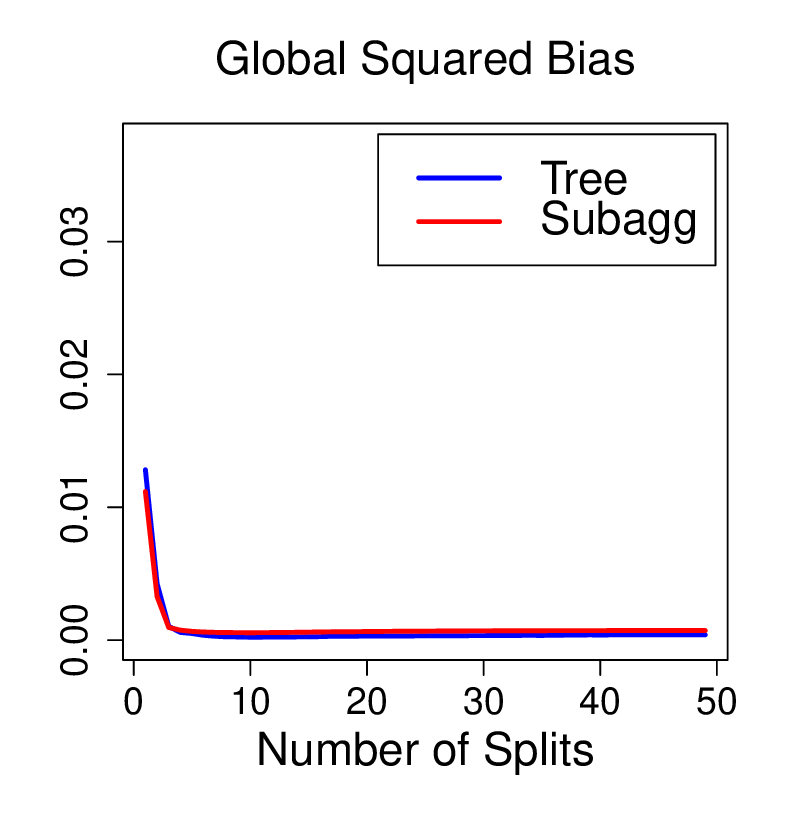}
\end{subfigure}
\begin{subfigure}{0.325\textwidth}
	\includegraphics[width=\textwidth]{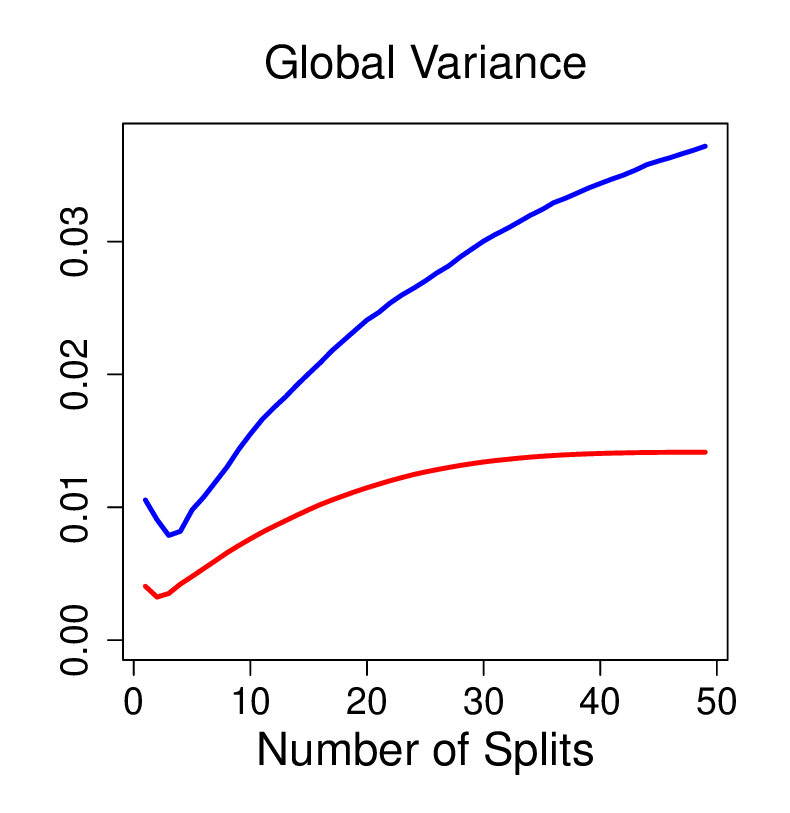}
\end{subfigure}
\begin{subfigure}{0.325\textwidth}
	\includegraphics[width=\textwidth]{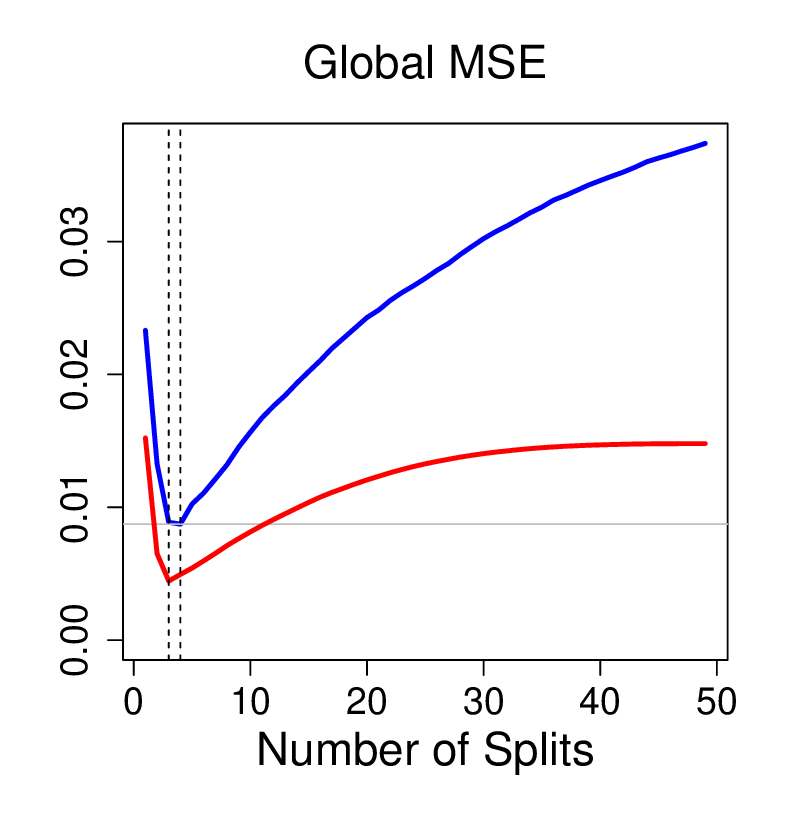}
\end{subfigure}
\caption{Global performance of trees (blue) and subagged trees (red) 
 as a function of the number of splits.}
\label{sub_bagging_large_trees}
\end{figure}

\subsection{Optimal Number of Splits as a Function of the Dataset Size}
We showed that for $n=100$ observations, the optimal number of splits in terms of mean-squared error is much lower than the maximum possible, both for trees and subagging ($N=4$ and $N=3$ respectively). This seems to still hold for larger dataset sizes. Figure \ref{opt_nb_splits} shows the optimal number of splits and corresponding global mean-squared error for trees and subagging as a function of $n$. We observe a roughly linear relation. This is of practical importance, as when the sample size is large, subagging trees of different sizes can become computationally costly. Instead, we can grow a single tree at different sizes, find the optimal number of splits with e.g. cross-validation, then use subagging with that same size for prediction, or, alternatively, perform cross-validation to choose the best number of splits \textit{around} $N=3$, e.g., to choose between 2,3 or 4 splits. Moreover, if interpretability of the model at hand is important, and if the difference in performance between a tree at its best and subagging at its best is not very large, then the single tree could be preferred. 
\begin{figure}[h]
\begin{subfigure}{0.325\textwidth}
	\includegraphics[width=\textwidth]{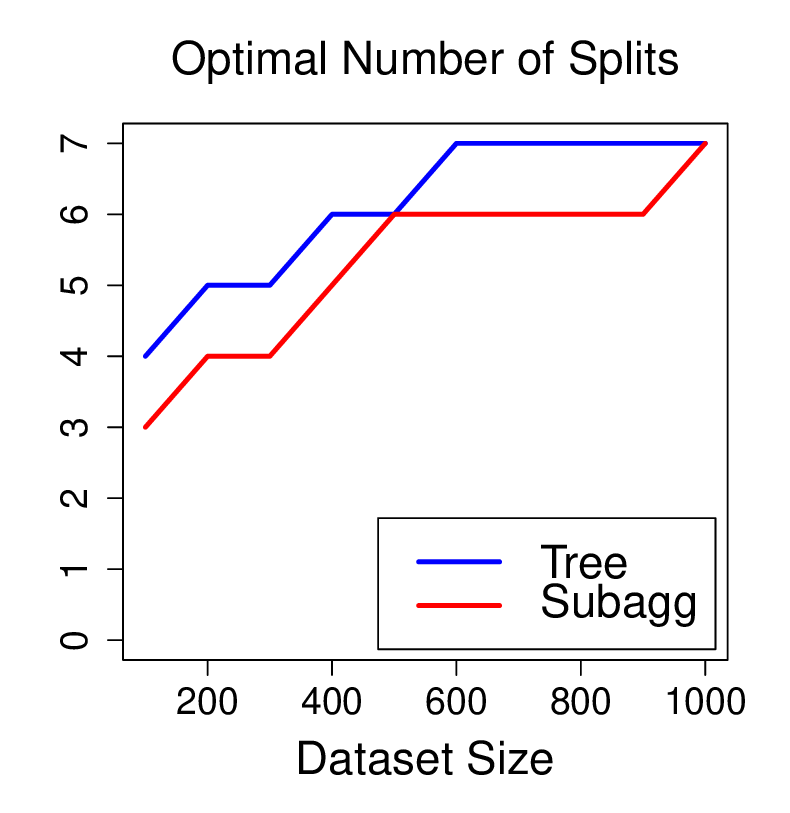}
\end{subfigure}
\begin{subfigure}{0.325\textwidth}
	\includegraphics[width=\textwidth]{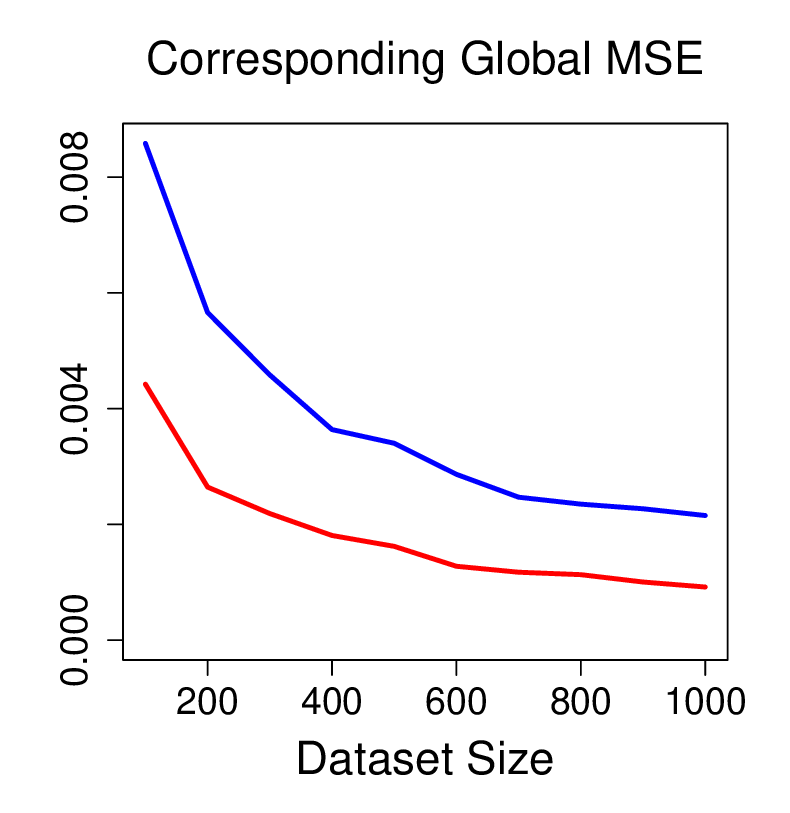}
\end{subfigure}
\caption{Optimal number of splits (left) as a function of $n$ and corresponding global mean-squared error (right) for trees (blue) versus subagging (red).}
\label{opt_nb_splits}
\end{figure}

\section{Robustness} \label{section_discussion}
Here we test the robustness of our results and implementations.

\subsection{Subsample Size and Replacement}\label{subsection_extension_bagging}
So far we have considered subagging instead of bagging, and fixed subsample size to $k=0.5n$. In Figure \ref{discussion_subsamples}, we show that our main empirical finding, namely that of Figure \ref{sub_bagging_large_trees}, still holds for different subsample sizes $k$ as well as if we do bagging instead of subagging. In all cases, the conclusion is the same: the optimal number of splits is around 3. A decrease in $k$ shows a decrease in variance but a bias is introduced, and vice-versa; increasing $k$ brings the subagged estimator closer to the tree\footnote{ These observations were already made by \citealp{buhlmann2002analyzing}, who pointed out that the subsample size can also be interpreted as a ``smoothing" parameter - see their Section 3.2.}. Moreover, bagging indeed behaves similarly to subagging with $k=0.5n$\footnote{Also in line with previous literature, e.g., \citealp{buhlmann2002analyzing}; \citealp{buja2006observations}.}. 

\begin{figure}[h]
\begin{subfigure}{0.325\textwidth}
	\includegraphics[width=\textwidth]{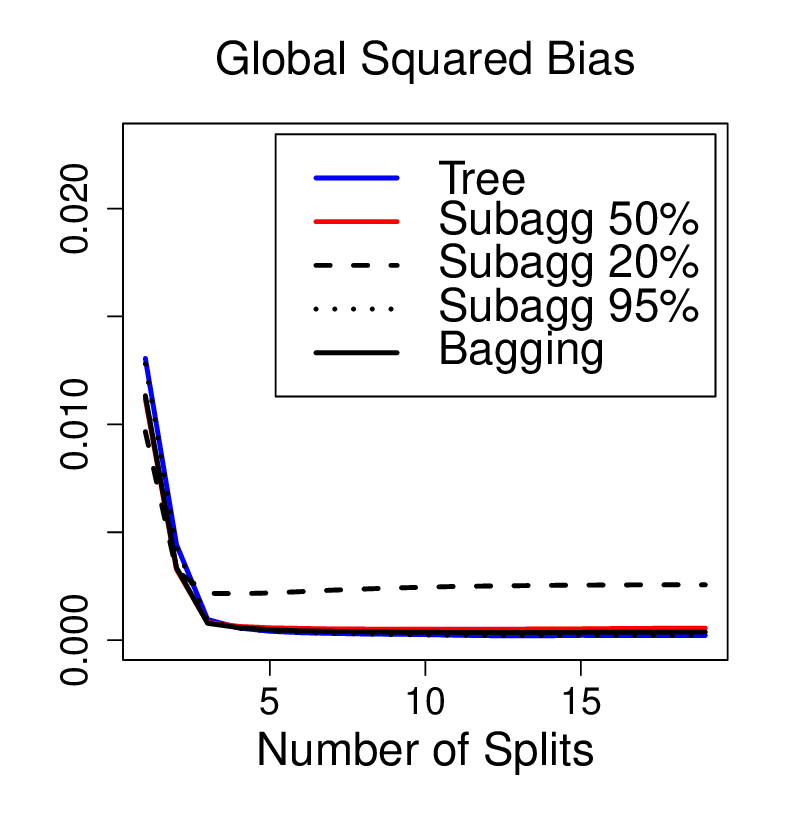}
\end{subfigure}
\begin{subfigure}{0.325\textwidth}
	\includegraphics[width=\textwidth]{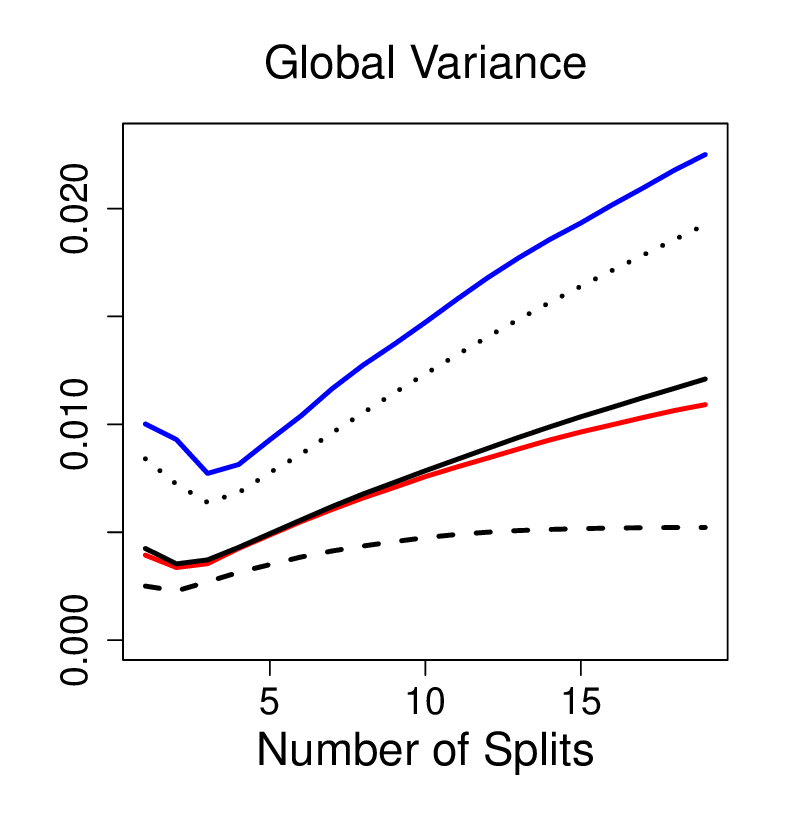}
\end{subfigure}
\begin{subfigure}{0.325\textwidth}
	\includegraphics[width=\textwidth]{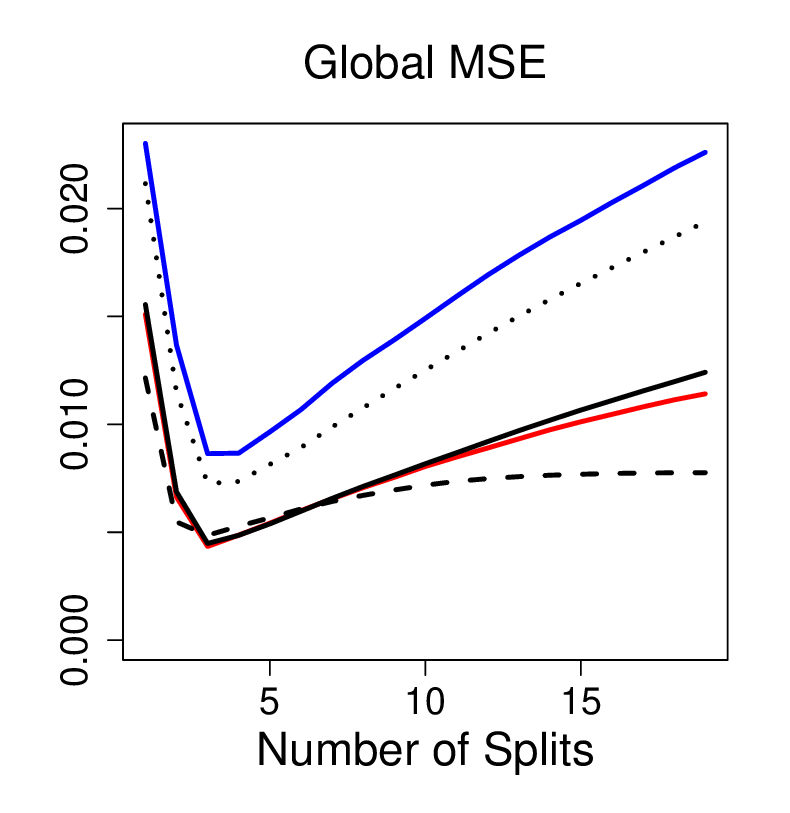}
\end{subfigure}
\caption{Global performance of trees (blue), subagged trees with $k=50\%n$ (red), $k=20\%n$ (dashed), $k=95\%n$ (dotted), and bagging (solid black).}
\label{discussion_subsamples}
\end{figure}

\subsection{Readily Available Implementations}\label{subsection_R_implem}
So far we have used our own implementations of recursive partitioning. Here we show that our main empirical finding, namely that of Figure \ref{sub_bagging_large_trees}, still holds if, instead of using our second implementation, we use the readily-available \Rlogo \ package \textbf{randomForest}\footnote{\citealp{randomForestR}, version 4.7-1.1.}, which incorporates both trees and subagging as special cases. We set the number of trees equal to 1 (\textbf{ntree = 1}) and sample without replacement $n=100$ observations (\textbf{replace =} FALSE and \textbf{sampsize = 100}) in order to obtain a tree. Then we set the number of trees to $B=50$ (\textbf{ntree = 50}) and the subsample size to $k=0.5n$ (\textbf{sampsize = 50}) to obtain subagging. We can control the \textit{maximum} number of splits by controlling the maximum number of nodes allowed, implemented with the option \textbf{maxnodes}: setting e.g. \textbf{maxnodes = 5} guarantees that we will grow trees with a maximum of 4 splits. We also set the minimum number of observations allowed in a terminal cell to one (\textbf{nodesize = 1}) and kept the default choices for all other options. 
Figure \ref{discussion_R_algo} shows the performance of a tree versus subagging as a function of the maximum number of splits allowed. We obtain best performance at a maximum of 3 splits for both methods (vertical dotted line in the right plot), resulting in mean-squared errors of 0.87\% (horizontal grey line) for the tree and 0.41\% for subagging, in line with our results: subagging large trees performs worse than a single tree at optimal size. 
\begin{figure}[h]
\begin{subfigure}{0.325\textwidth}
	\includegraphics[width=\textwidth]{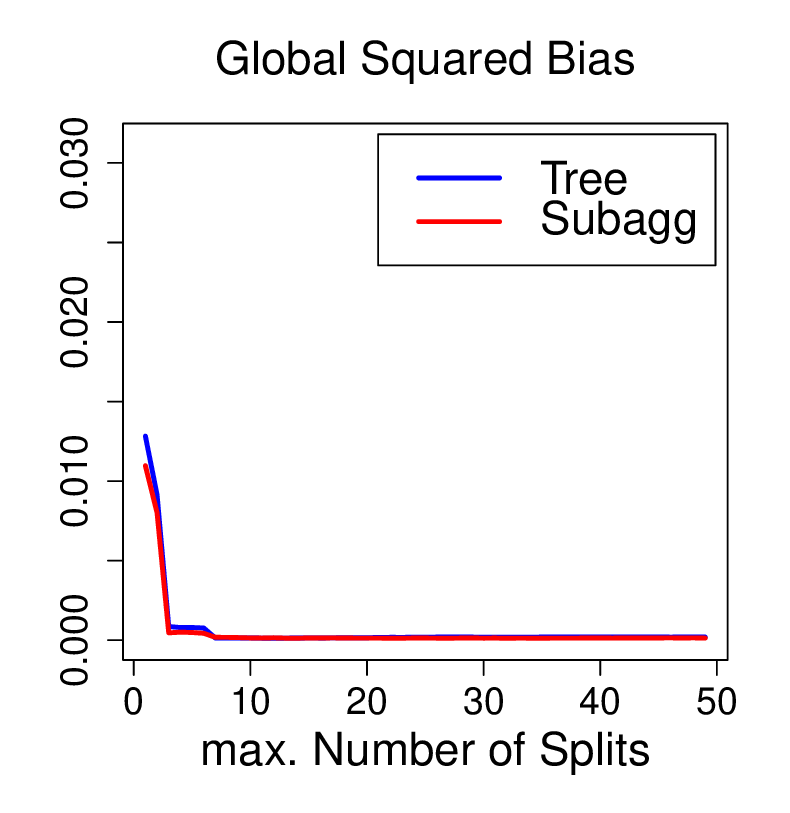}
\end{subfigure}
\begin{subfigure}{0.325\textwidth}
	\includegraphics[width=\textwidth]{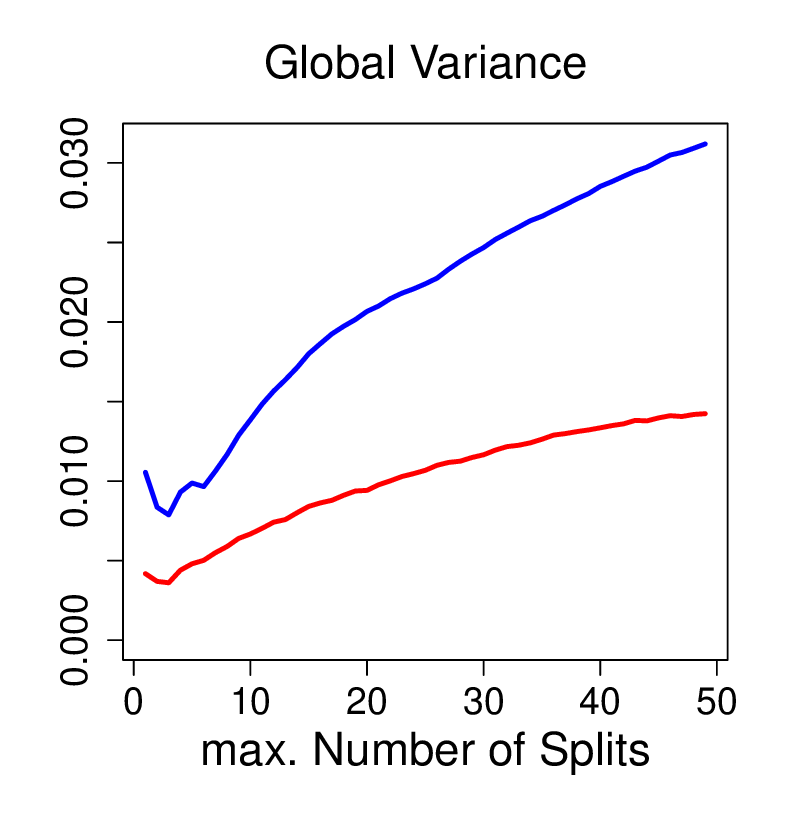}
\end{subfigure}
\begin{subfigure}{0.325\textwidth}
	\includegraphics[width=\textwidth]{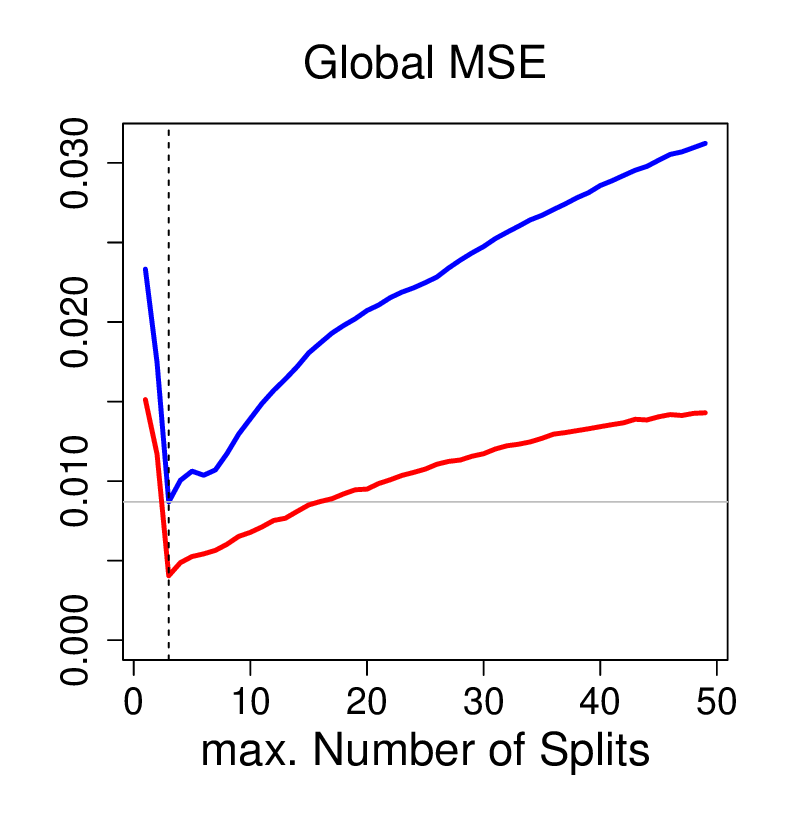}
\end{subfigure}
\caption{Performance of trees versus subagging using the \textbf{randomForest} package.} 
\label{discussion_R_algo}
\end{figure}

Finally, in Figure \ref{consistency_figure_replication} we show that the results of Figure \ref{consistency_figure} based on our first implementation can also be obtained using the \textbf{randomForest} package. Again, we set the number of trees equal to 1 (\textbf{ntree = 1}) and sample without replacement $n=2000$ observations (\textbf{replace =} FALSE and \textbf{sampsize = 2000}) in order to obtain a single tree. Then, we used \textbf{maxnodes = $\frac{n}{h_n}$} as a \textit{proxy} to control for the number of observations in cells in scenario b) (so \textbf{maxnodes = $\frac{n}{n^{0.65}}$}); for small trees, we set \textbf{maxnodes = 2}; for large trees, we set \textbf{maxnodes = $\frac{n}{5}$}. In all cases we kept the minimum \textbf{nodesize = 1} and we used the default values for all other parameters. The bias-variance trade-off is again illustrated.
\begin{figure}[h]
\begin{subfigure}{0.325\textwidth}
	\includegraphics[width=\textwidth]{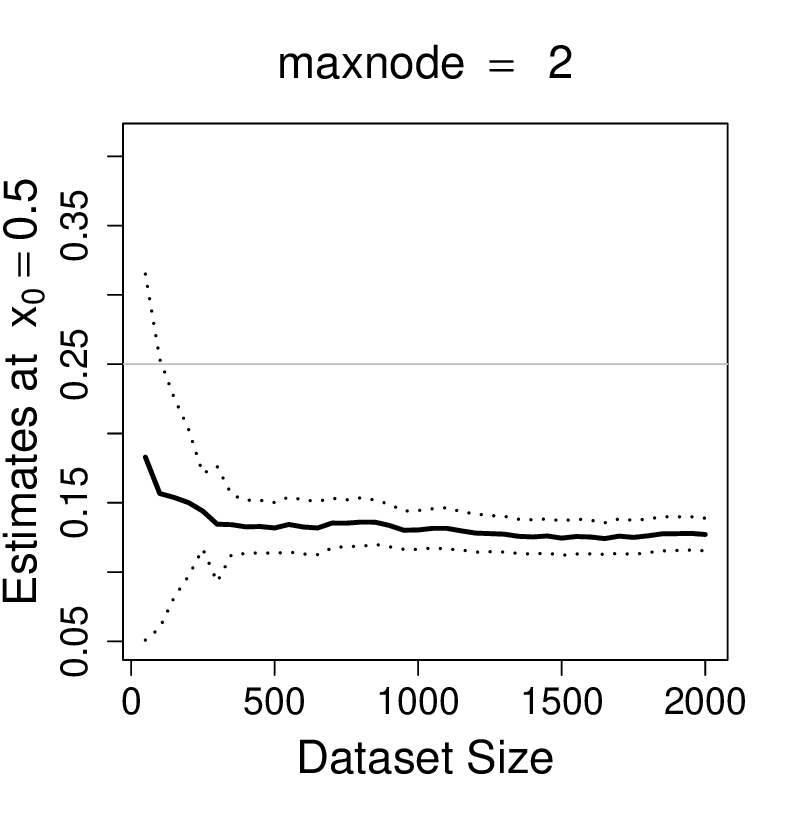}
\end{subfigure}
\begin{subfigure}{0.325\textwidth}
	\includegraphics[width=\textwidth]{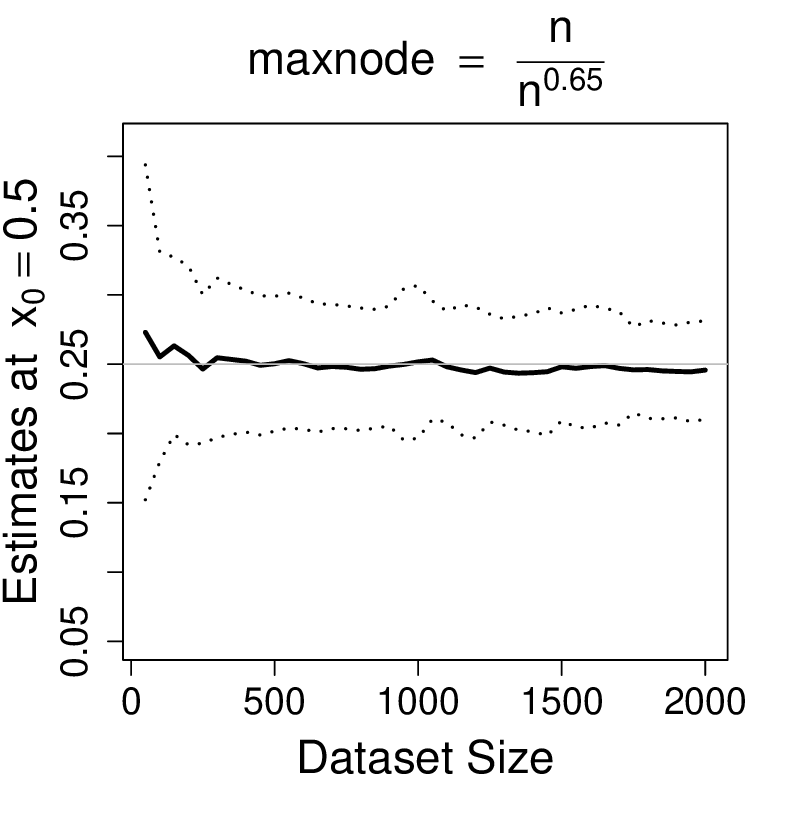}
\end{subfigure}
\begin{subfigure}{0.325\textwidth}
	\includegraphics[width=\textwidth]{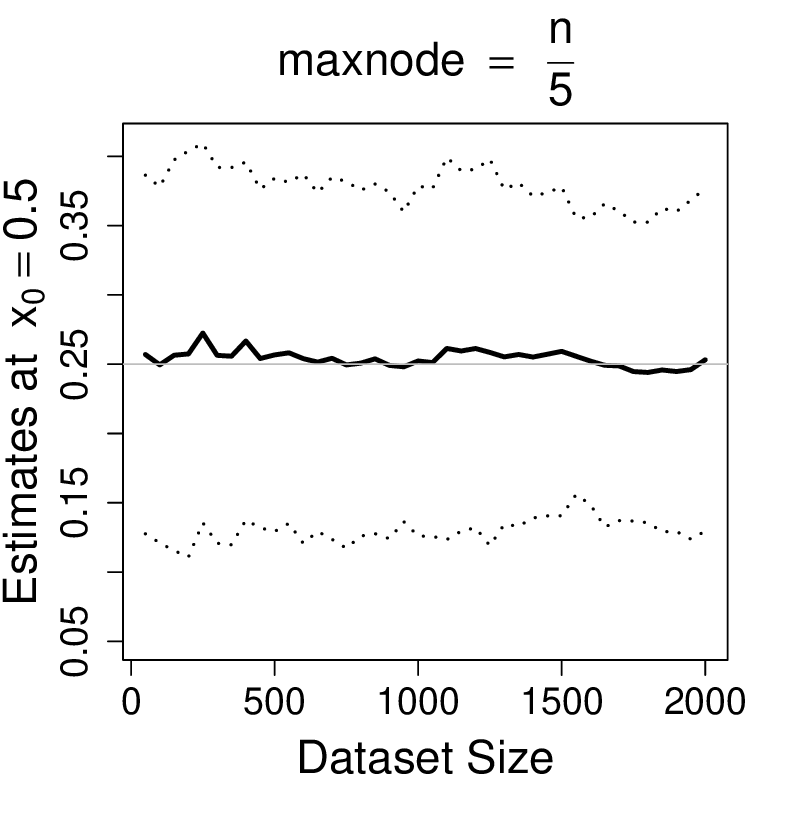}
\end{subfigure}
\caption{(In)consistency of trees using the \textbf{randomForest} package.}
\label{consistency_figure_replication}
\end{figure}

\section{Conclusion}\label{conclusion}
We studied CART and subagging in the context of regression estimation and have shown that tree size plays an important role in the performance of these methods.
We established sufficient conditions for point-wise consistency of trees and provided an algorithm that satisfies them.
Based on this, we formalized and illustrated the bias-variance trade-off associated with tree size. Then, we studied the effect of subagging on trees of varying sizes. We have illustrated the effect on weights and showed in simulations how subagging performs, compared to a single tree, when the number of splits is increased, both for the single tree and the trees constituting the subagged estimator. We found that a single tree optimally grown can outperform subagging if its subtrees are large. For practical applications, our findings suggest that large trees should not be the default choice in ensemble methods, and computational time can be saved upon by first finding the optimal size for a single tree in the problem at hand before building the ensemble.

\appendix

\section{Proofs}\label{appendix}

\subsection{Proof of Proposition \ref{proposition1}}

Consider 
\begin{equation}
\mathcal{C} = \mathbb{V}[Y|X\in C]-\frac{\mathbb{P}(X\in C_L)}{\mathbb{P}(X\in C)}\mathbb{V}[Y|X\in C_L]-\frac{\mathbb{P}(X\in C_R)}{\mathbb{P}(X\in R)}\mathbb{V}[Y|X\in C_R]
\label{temp_proof_1}
\end{equation}
as in (\ref{CART_criterion}). We have $\mathbb{V}(Y|X\in C) = \mathbb{E}[Y^2|X\in C] - \mathbb{E}[Y|X\in C]^2$ and similarly for $C_L$ and $C_R$. Using that $\mathbbm{1}(X\in C)=\mathbbm{1}(X\in C_L)+\mathbbm{1}(X\in C_R)$ and the linearity of the expectation we get 
\begin{equation}
\begin{split}
&\mathbb{P}(X\in C) \mathbb{V}(Y|X\in C)=\mathbb{E}[Y^2\mathbbm{1}(X\in C_L)]+\mathbb{E}[Y^2\mathbbm{1}(X\in C_R)]\\
&-\frac{1}{\mathbb{P}(X\in C)}\Big\{\mathbb{E}[Y\mathbbm{1}(X\in C_L)]+\mathbb{E}[Y\mathbbm{1}(X\in C_R)]\Big\}^2.
\end{split}
\end{equation}

We also have
\begin{equation}
\mathbb{P}(X\in C_L)\mathbb{V}(Y|X\in C_L)=\mathbb{E}[Y^2\mathbbm{1}(X\in C_L)]-\frac{\mathbb{E}[Y\mathbbm{1}(X\in C_L)]^2}{\mathbb{P}(X\in C_L)}
\end{equation}
and similarly for $C_R$. Plugging everything into (\ref{temp_proof_1}), the terms containing $Y^2$ cancel out and we are left with
\begin{equation}
\begin{split}
\mathcal{C}&=\mathbb{E}[Y\mathbbm{1}(X\in C_L)]^2\Bigg(\frac{1}{\mathbb{P}(X\in C_L)}-\frac{1}{\mathbb{P}(X\in C)}\Bigg)\\
&\ \ \ \ -2\mathbb{E}[Y\mathbbm{1}(X\in C_L)]\mathbb{E}[Y\mathbbm{1}(X\in C_R)]\frac{1}{\mathbb{P}(X\in C)}\\
&\ \ \ \ +\mathbb{E}[Y\mathbbm{1}(X\in C_R)]^2\Bigg(\frac{1}{\mathbb{P}(X\in C_R)}-\frac{1}{\mathbb{P}(X\in C)}\Bigg)\\
&=\frac{\mathbb{P}(X\in C_L)\mathbb{P}(X\in C_R)}{\mathbb{P}(X\in C)}\Big\{\mathbb{E}[Y|X\in C_L]-\mathbb{E}[Y|X\in C_R]\Big\}^2
\end{split}
\end{equation}
which concludes the proof of Proposition \ref{proposition1}.

\subsection{Proof of Proposition \ref{proposition2}}

Let $(X_{\iota_1},Y_{\iota_1}), \dots, (X_{\iota_\nu},Y_{\iota_\nu})$ be $\nu\geq3$ re-ordered observations constituting a cell $C$. Suppose that the criterion is constant for any split in $C$. Using Proposition \ref{proposition1}, this means that there exists a constant $K_0\geq0$ such that for all $\kappa\in\{1,\dots,\nu-1\}$, we have
\begin{equation}
\frac{\kappa(\nu-\kappa)}{\nu}\Bigg(\frac{1}{\kappa}\sum_{\lambda=1}^\kappa Y_{\iota_\lambda}-\frac{1}{\nu-\kappa}\sum_{\lambda=\kappa+1}^\nu Y_{\iota_\lambda}\Bigg)^2=nK_0.
\label{L_constant_meaning}
\end{equation}
In particular (\ref{L_constant_meaning}) is true for $\kappa=1$ and $\kappa=\nu-1$, which, combined, give
\begin{equation}
(\nu-1)Y_{\iota_1}-\sum_{\lambda=2}^{\nu}Y_{\iota_\lambda}=\pm\Bigg(\sum_{\lambda=1}^{\nu-1}Y_{\iota_\lambda}-(\nu-1)Y_\nu\Bigg).
\label{plus_or_minus}
\end{equation}
Suppose that in (\ref{plus_or_minus}) the ``+" holds. This is equivalent (recall $\nu\geq3$) to 
\begin{equation}
\frac{Y_{\iota_1}+Y_{\iota_\nu}}{2}=\frac{1}{\nu-2}\sum_{\lambda=2}^{\nu-1}Y_{\iota_\lambda}.
\end{equation}
Plug this into (\ref{L_constant_meaning}) with $\kappa=1$. After simplifying, we obtain
\begin{equation}
Y_{\iota_1}=Y_{\iota_\nu}\pm2\sqrt{\frac{\nu-1}{\nu}nK_0}
\end{equation} 
which is almost surely impossible. Next suppose that in (\ref{plus_or_minus}) the ``-" holds. This is equivalent to $Y_{\iota_1}=Y_{\iota_\nu}$ which is also almost surely impossible as long as $Y$ is a continuous random variable. Therefore, by contradiction, almost surely the criterion is not constant in any cell containing at least three observations. This concludes the proof of Proposition \ref{proposition2}.

\subsection{Proof of Theorem \ref{consistency_theorem}}
We start by re-formulating the honesty assumption. We assume given two datasets:
\begin{enumerate}
\item $D_n=\{(X_1,Y_1),\dots,(X_n,Y_n)\}$, a dataset used to partition $[0,1]^p$, and
\item $D_n'=\{(X_1',Y_1'),\dots,(X_n',Y_n')\}$, a dataset used for estimation, independent of $D_n$, of same size as $D_n$.
\end{enumerate}

\noindent Fix $x_0\in[0,1]^p$. We note $C_n=C_n(x_0;D_n)$ the terminal cell containing $x_0$ in the partition obtained using the first dataset. Throughout the proof, we also get rid of the dependence of weights on $x_0$. The tree estimator is noted 

\begin{equation}
T_n(x_0) = \sum_{i=1}^n W_{n,i} Y_i' \text{\ \ with \ } W_{n,i} = \frac{\mathbbm{1}_{X_i'\in C_n}}{\sum_{j=1}^n\mathbbm{1}_{X_j'\in C_n}}
\end{equation}
which can be re-written as
\begin{equation}
T_n(x_0) = \sum_{i=1}^n W_{n,i} f(X_i') + \sum_{i=1}^n W_{n,i} \varepsilon_i'.
\end{equation}
We treat the two terms separately.

\subsubsection{Bias of the error term}
\begin{lemma}
For all $n$, 
\begin{equation}
\mathbb{E}\Bigg[\sum_{i=1}^nW_{n,i}\varepsilon_i'\Big|\mathbbm{1}(X_1'\in C_n),\dots,\mathbbm{1}(X_n'\in C_n)\Bigg]=0.
\label{conditional_error_equation_scenario2}
\end{equation}
\label{error_bias_scenario2}
\end{lemma}

\begin{proof}
From Assumptions \ref{assumption1} and \ref{assumption2},
\begin{equation*}
 \forall i, \ \mathbb{E}[\varepsilon_i'|\mathbbm{1}(X_1'\in C_n),\dots,\mathbbm{1}(X_n'\in C_n)]=\mathbb{E}[\varepsilon_i']=0.
\end{equation*}
\end{proof}

\subsubsection{Variance of the error term}

\begin{lemma}
For all $n$,
\begin{equation}
\mathbb{V}\Bigg[\sum_{i=1}^nW_{n,i}\varepsilon_i'\Big|\mathbbm{1}(X_1'\in C_n),\dots,\mathbbm{1}(X_n'\in C_n)\Bigg]=\frac{\sigma^2}{\sum_{i=1}^n\mathbbm{1}(X_{i}'\in C_{n})}.
\end{equation}
\label{error_variance}
\end{lemma}

\begin{proof}
\begin{equation*}
\begin{split}
&\mathbb{V}\Bigg[\sum_{i=1}^nW_{n,i}\varepsilon_i'\Big|\mathbbm{1}(X_1'\in C_n),\dots,\mathbbm{1}(X_n'\in C_n)\Bigg]\\
&=\mathbb{E}\Bigg[\Big(\sum_{i=1}^nW_{n,i}\varepsilon_i'\Big)^2\Big|\mathbbm{1}(X_1'\in C_n),\dots,\mathbbm{1}(X_n'\in C_n)\Bigg]-0^2\text{ \ from (\ref{conditional_error_equation_scenario2})}\\
&=\sum_{i=1}^nW_{n,i}^2\mathbb{E}[\varepsilon_i'^2|\mathbbm{1}(X_1'\in C_n),\dots,\mathbbm{1}(X_n'\in C_n)]\\
&+2\sum_{j\neq i}W_{n,i}W_{n,j}\mathbb{E}[\varepsilon_i'\varepsilon_j'|\mathbbm{1}(X_1'\in C_n),\dots,\mathbbm{1}(X_n'\in C_n)]\\
&=\sigma^2\sum_{i=1}^nW_{n,i}^2+0\\
&=\frac{\sigma^2}{\sum_{i=1}^n\mathbbm{1}(X_{i}'\in C_n)}.
\end{split}
\end{equation*}
\end{proof}

\begin{lemma}
\begin{equation}
\frac{1}{\sum_{i=1}^n\mathbbm{1}(X_{i}'\in C_{n})}\xrightarrow[n\rightarrow\infty]{\mathbb{P}}0.
\end{equation}
\end{lemma}
\begin{proof}
Conditionally on $C_n$, $\sum_{i=1}^n\mathbbm{1}(X_{i}'\in C_{n})$ follows a Binomial distribution of parameters $(n,\lambda(C_n))$. Therefore it is sufficient to show 
\begin{equation}
n\ \lambda(C_n)\xrightarrow[n\rightarrow\infty]{\mathbb{P}}\infty.
\end{equation}
We know from Assumption \ref{assumption3} 
\begin{equation}
n\ \lambda_n(C_n)\geq h_n \xrightarrow[n\rightarrow\infty]{}\infty
\end{equation}
where 
\begin{equation}
\lambda_n(C_n):=\frac{1}{n}\sum_{i=1}^n\mathbbm{1}(X_{i}\in C_{n}).
\end{equation}
From empirical process theory, the set of hyper-rectangles is a Glivenko-Cantelli class (of VC dimension $2 p<\infty$), note it $\mathcal{R}$, and we have 
\begin{equation}
\sqrt{n}\sup_{C\in\mathcal{R}}|\lambda(C)-\lambda_n(C)|=\mathcal{O}_\mathbb{P}(1)
\end{equation}
and for any realization of $C_n$, 
\begin{equation}
\sqrt{n}|\lambda(C_n)-\lambda_n(C_n)|\leq \sqrt{n}\sup_{C\in\mathcal{R}}|\lambda(C)-\lambda_n(C)|
\end{equation}
therefore, using the decomposition
\begin{equation}
\sqrt{n}\lambda(C_n)=\sqrt{n}\lambda_n(C_n)+\sqrt{n}(\lambda(C_n)-\lambda_n(C_n))
\end{equation}
we get that if $\sqrt{n}\lambda_n(C_n)$ tends to infinity, then so does $\sqrt{n}\lambda(C_n)$ for any realization of $C_n$. By Assumption \ref{assumption3}, 
\begin{equation}
n\lambda_n(C_n)\geq h_n= n^\alpha
\end{equation}
therefore
\begin{equation}
\sqrt{n}\lambda_n(C_n)\geq n^{\alpha+\frac{1}{2}-1}
\end{equation}
with $\alpha+\frac{1}{2}-1>0$ because $\alpha>\frac{1}{2}$. Therefore $\sqrt{n}\lambda_n(C_n)\xrightarrow[n\rightarrow\infty]{\mathbb{P}}\infty$, hence $\sqrt{n}\lambda(C_n)\xrightarrow[n\rightarrow\infty]{}\infty$ and hence also $n\lambda(C_n)\xrightarrow[n\rightarrow\infty]{}\infty$ for any realization of $C_n$. 
\end{proof}

\subsubsection{Bias of the regression term}

\begin{lemma}
For all $n$,
\begin{equation}
\mathbb{E}\Bigg[\sum_{i=1}^nW_{n,i}f(X_i')|\mathbbm{1}(X_1'\in C_{n}),\dots,\mathbbm{1}(X_n'\in C_{n}),C_{n}\Bigg]=\mathbb{E}[f(X)|X\in C_{n}].
\end{equation}
\label{regression_bias}
\end{lemma}

\begin{proof}
Conditionally on $\{\mathbbm{1}(X_i'\in C_{n}),C_{n}\}$, $f(X_i')$ is independent of $\mathbbm{1}(X_j'\in C_{n})$ for all $j\neq i$, i.e.,
\begin{equation}
\mathbb{E}[f(X_i')|\mathbbm{1}(X_1'\in C_{n}),\dots,\mathbbm{1}(X_n'\in C_{n}),C_{n}]=\mathbb{E}[f(X_i')|\mathbbm{1}(X_i'\in C_{n}),C_{n}]. 
\end{equation}
Moreover
\begin{equation*}
\begin{split}
&\mathbb{E}[f(X_i')|\mathbbm{1}(X_i'\in C_{n}),C_{n}]\\
&=\mathbbm{1}(X_i'\in C_{n})\mathbb{E}[f(X)|X\in C_{n}]+\mathbbm{1}(X_i'\notin C_{n})\mathbb{E}[f(X)|X\notin C_{n}].
\end{split}
\end{equation*}
\end{proof}

\subsubsection{Variance of the regression term}

\begin{lemma}
For all $n$,
\begin{equation}
\mathbb{V}\Bigg[\sum_{i=1}^nW_{n,i}f(X_i')|\mathbbm{1}(X_1'\in C_{n}),\dots,\mathbbm{1}(X_n'\in C_{n}),C_{n}\Bigg]=\mathbb{V}[f(X)|X\in C_{n}].
\end{equation}
\label{regression_variance}
\end{lemma}

\begin{proof}
From Lemma \ref{regression_bias}, we have
\begin{equation}
\begin{split}
&\mathbb{V}\Bigg[\sum_{i=1}^nW_{n,i}f(X_i)|\mathbbm{1}(X_1\in C_{n}),\dots,\mathbbm{1}(X_n\in C_{n}),C_{n}\Bigg]\\
&=\mathbb{E}\Bigg[\Big(\sum_{i=1}^nW_{n,i}f(X_i')\Big)^2|\mathbbm{1}(X_1'\in C_{n}),\dots,\mathbbm{1}(X_n'\in C_{n}),C_{n}\Bigg]-\mathbb{E}[f(X)|X\in C_{n}]^2
\end{split}
\end{equation}
where
\begin{equation}
\begin{split}
&\mathbb{E}\Bigg[\Big(\sum_{i=1}^nW_{n,i}f(X_i')\Big)^2|\mathbbm{1}(X_1'\in C_{n}),\dots,\mathbbm{1}(X_n'\in C_{n}),C_{n}\Bigg]\\
&=\sum_{i=1}^nW_{n,i}^2\mathbb{E}[f(X_i')^2|\mathbbm{1}(X_1'\in C_{n}),\dots,\mathbbm{1}(X_n'\in C_{n}),C_{n}]\\
&+2\sum_{i\neq j}W_{n,i}W_{n,j}\mathbb{E}[f(X_i')f(X_j')|\mathbbm{1}(X_1'\in C_{n}),\dots,\mathbbm{1}(X_n'\in C_{n}),C_{n}].
\end{split}
\end{equation}
For all $i$, we have
\begin{equation}
\mathbb{E}[f(X_i')^2|\mathbbm{1}(X_1'\in C_{n}),\dots,\mathbbm{1}(X_n'\in C_{n}),C_{n}]=\mathbb{E}[f(X_i')^2|\mathbbm{1}(X_i'\in C_{n}),C_{n}]
\end{equation}
and for all $i\neq j$, we have
\begin{equation}
\begin{split}
\mathbb{E}[f(X_i')f(X_j')|\mathbbm{1}(X_1'\in C_{n}),&\dots,\mathbbm{1}(X_n'\in C_{n}),C_{n}]=\\
&\mathbb{E}[f(X_i')f(X_j')|\mathbbm{1}(X_i'\in C_{n}),\mathbbm{1}(X_j'\in C_{n}),C_{n}].
\end{split}
\end{equation}
Moreover,
\begin{equation}
\begin{split}
\mathbb{E}[f(X_i')^2|\mathbbm{1}(X_i'\in C_{n})]&=\mathbbm{1}(X_i'\in C_{n})\mathbb{E}[f(X)^2|X\in C_{n})]\\
&+\mathbbm{1}(X_i'\notin C_{n})\mathbb{E}[f(X)^2|X\notin C_{n})]
\end{split}
\end{equation}
and
\begin{equation}
\begin{split}
\mathbb{E}[f(X_i')f(X_j')&|\mathbbm{1}(X_i'\in C_{n}),\mathbbm{1}(X_j'\in C_{n})]=\\
&\mathbbm{1}(X_i'\in C_{n})\mathbbm{1}(X_j'\in C_{n})(\mathbb{E}[f(X)|X\in C_n)])^2\\
+&\mathbbm{1}(X_i'\in C_{n})\mathbbm{1}(X_j'\notin C_{n})\mathbb{E}[f(X)|X\in C_{n})]\mathbb{E}[f(X)|X\notin C_{n})]\\
+&\mathbbm{1}(X_i'\notin C_{n})\mathbbm{1}(X_j'\in C_{n})\mathbb{E}[f(X)|X\notin C_{n})]\mathbb{E}[f(X)|X\in C_{n})]\\
+&\mathbbm{1}(X_i'\notin C_{n})\mathbbm{1}(X_j'\notin C_{n})(\mathbb{E}[f(X)|X\notin C_{n})])^2.
\end{split}
\end{equation}
Note that
\begin{itemize}
\item $W_{n,i}\mathbbm{1}(X_i'\in C_{n})=1$ for all $i$,
\item $W_{n,i}\mathbbm{1}(X_j'\in C_{n})=0$ whenever $i\neq j$,
\item $\sum_{i=1}^nW_{n,i}^2=\frac{1}{\sum_{i=1}^n\mathbbm{1}(X_i'\in C_n)}$, and
\item $2\sum_{i\neq j}W_{n,i}W_{n,j}=\Big(\sum_{i=1}^n W_{n,i}\Big)^2-\sum_{i=1}^n W_{n,i}^2=1-\frac{1}{\sum_{i=1}^n\mathbbm{1}(X_i'\in C_n)}$.
\end{itemize}
Therefore 
\begin{equation}
\begin{split}
&\sum_{i=1}^nW_{n,i}^2\mathbb{E}[f(X_i')^2|\mathbbm{1}(X_1'\in C_{n}),\dots,\mathbbm{1}(X_n'\in C_{n}),C_{n}]\\
&=\mathbb{E}[f(X)^2|X\in C_n]\frac{1}{\sum_{i=1}^n\mathbbm{1}(X_i'\in C_n)}
\end{split}
\end{equation}
and
\begin{equation}
\begin{split}
&2\sum_{i\neq j}W_{n,i}W_{n,j}\mathbb{E}[f(X_i')f(X_j')|\mathbbm{1}(X_1'\in C_{n}),\dots,\mathbbm{1}(X_n'\in C_{n}),C_{n}]\\
&=\mathbb{E}[f(X)^2|X\in C_n]\Bigg(1-\frac{1}{\sum_{i=1}^n\mathbbm{1}(X_i'\in C_n)}\Bigg).
\end{split}
\end{equation}
Putting everything together gives
\begin{equation}
\begin{split}
&\mathbb{V}\Bigg[\sum_{i=1}^nW_{n,i}f(X_i')|\mathbbm{1}(X_1'\in C_{n}),\dots,\mathbbm{1}(X_n'\in C_{n}),C_{n}\Bigg]\\
&=\mathbb{E}[f(X)^2|X\in C_n]\Bigg(\frac{1}{\sum_{i=1}^n\mathbbm{1}(X_i'\in C_n)}+1-\frac{1}{\sum_{i=1}^n\mathbbm{1}(X_i'\in C_n)}\Bigg)\\
&\ \ \ -\mathbb{E}[f(X)|X\in C_{n}]^2\\
&=\mathbb{V}[f(X)|X\in C_{n}].
\end{split}
\end{equation}

\end{proof}

\subsubsection{Putting everything together}

From Lemmas 2 and 3 we get
\begin{equation}
\mathbb{V}\Bigg[\sum_{i=1}^nW_{n,i}\varepsilon_i'|\mathbbm{1}(X_1'\in C_n),\dots,\mathbbm{1}(X_n'\in C_n)\Bigg]\xrightarrow[n\rightarrow\infty]{\mathbb{P}}0
\end{equation}
and with Lemma 1 we obtain 
\begin{equation}
\sum_{i=1}^nW_{n,i}\varepsilon_i'\xrightarrow[n\rightarrow\infty]{\mathbb{P}}0.
\end{equation}
From Lemma 5 and Assumption \ref{assumption4} we get
\begin{equation}
\mathbb{V}\Bigg[\sum_{i=1}^nW_{n,i}f(X_i')|\mathbbm{1}(X_1'\in C_{n}),\dots,\mathbbm{1}(X_n'\in C_{n}),C_{n}\Bigg]\xrightarrow[n\rightarrow\infty]{\mathbb{P}}0
\end{equation}
and with Lemma 4 we obtain
\begin{equation}
\sum_{i=1}^nW_{n,i}f(X_i')\xrightarrow[n\rightarrow\infty]{\mathbb{P}}f(x_0).
\end{equation}
Therefore using Slutsky's theorem
\begin{equation}
\hat{T}_{n}(x_0)\xrightarrow[n\rightarrow\infty]{\mathbb{P}}f(x_0)
\end{equation}
which concludes the proof of Theorem \ref{consistency_theorem}.

\vskip 0.2in
\bibliography{sample}

\end{document}